\documentclass[twoside]{article}
\usepackage[accepted]{aistats2012}

\usepackage{amsmath,amssymb,amsthm}
\usepackage{algorithm}
\usepackage{algorithmic}
\usepackage{longtable}
\usepackage{tabularx}
\usepackage{multirow}
\usepackage{graphicx}

\newcommand{\abs}[1]{\ensuremath {\left| #1 \right|}}
\newcommand{\norm}[1]{\ensuremath{\left\| #1 \right\|}}
\newcommand{\innerproduct}[2]{\ensuremath{\ \left\langle #1 , #2 \right\rangle}}

\newcommand{\inner}[1]{\left\langle#1\right\rangle}

\def\R{\mathbb{R}}

\def\vol{\mathop{\rm vol}\nolimits}

\def\argmin{\mathop{\rm arg\,min}\limits}%    a math operator.
\def\sign{\mathop{\rm sign}\limits}

\def\relaxed_opt{\mathrm{relaxed\_opt}}

\def\cut{\mathrm{cut}}
\def\Ncut{\mathrm{NCut}}

\def\ones{\mathbf{1}}
\def\Ones{\mathbb{I}}

\def\subpc{_{p,\chi}}

\def\subpc2{_{p,C}}

\def\Bcut{\mathrm{NCut}}

\def\bal{\mathrm{bal}}
\def\gvol{\mathrm{gvol}}
\def\neigh{\mathrm{neigh}}
\newtheorem{lemma}{Lemma}[section]

\newtheorem{theorem}{Theorem}[section]

\newtheorem{definition}{Definition}[section]

\begin{document}



\twocolumn[

\aistatstitle{Constrained 1-Spectral Clustering}

\aistatsauthor{ Syama Sundar Rangapuram \And Matthias Hein}

\aistatsaddress{ srangapu@mpi-inf.mpg.de\\ Max Planck Institute for Computer Science\\ Saarland University, Saarbruecken, Germany 
\And hein@cs.uni-saarland.de\\ Saarland University, Saarbruecken\\ Germany} 
]

\begin{abstract}
An important form of prior information in clustering comes in form of cannot-link and must-link constraints. 
We present a generalization of the popular spectral clustering technique which integrates such constraints.
Motivated by the recently proposed $1$-spectral clustering for the unconstrained problem, our method 
is based on a tight relaxation of the constrained normalized cut into a continuous optimization problem.  
Opposite to all other methods which have been suggested for constrained spectral clustering, we can always guarantee to
satisfy all constraints. Moreover, our soft formulation allows to optimize a trade-off between normalized cut and 
the number of violated constraints. An efficient implementation is provided which scales to large datasets. 
We outperform consistently all other proposed methods in the experiments.
%
%
%Our normalized graph cut problem underlying spectral clustering. 
%We show that the constrained normalized cut problem
%allows for an equivalent continuous formulation where one minimizes
% an unconstrained ratio of differences of convex functions. 
% We provide an efficient first-order scheme for the minimization of the resulting
%optimization problem. Moreover, we propose the construction of a reduced graph 
%which automatically satisfies all must-link constraints while preserving
%the values of all cuts satisfying the constraints. 
%Contrary to other state-of-the-art methods we can always guarantee to satisfy all constraints. 
%In the experiments
%we show that we outperform other methods both in terms of achieved cut and resulting clustering error.
\end{abstract}

\section{Introduction}

The task of clustering is to find a natural grouping of items given e.g. pairwise similarities.
In real world problems, such a natural grouping is often hard to discover with given similarities alone or there is more than one way to group the given items. In either case, clustering methods benefit from domain knowledge that gives bias to the desired clustering. Wagstaff et. al \cite{WagCarSch2001} are the first to consider constrained clustering by encoding available domain knowledge in the form of pairwise must-link (ML, for short) and cannot-link (CL) constraints. By incorporating these constraints into $k$-means they achieve much better performance. Since acquiring such constraint information is relatively easy, constrained clustering has become an active area of research; see \cite{BasDavWag2008} for an overview. 

Spectral clustering is a graph-based
%\footnote{In graph based methods, a similarity graph is constructed, 
%e.g., $k$-NN or $\epsilon$- neighborhood graph, from the pairwise similarities and clustering problem is transformed to finding an optimal cut subject to some balance criteria.} 
clustering algorithm originally derived as a relaxation of the NP-hard normalized cut problem. The spectral relaxation leads to an eigenproblem for the graph Laplacian, see \cite{HagKah1991, ShiMal2000, Lux2007}. However, it is known that the spectral relaxation can be quite loose \cite{GuaMil98}. More recently, it has been shown that one can equivalently rewrite the discrete (combinatorial) normalized Cheeger cut
problem into a continuous optimization problem using the nonlinear $1$-graph Laplacian \cite{HeiBue2010, SB10} which yields much better cuts than
the spectral relaxation.
In further work it is shown that even all balanced graph cut problems, including normalized cut, have a tight relaxation into a continuous 
optimization problem \cite{HeiSet2011}.

The first approach to integrate constraints into spectral clustering was based on the idea of modifying the 
weight matrix in order to enforce the must-link and cannot-link constraints and then solving the resulting unconstrained problem \cite{KamKleMan2003}. 
Another idea is to adapt the embedding obtained from the first $k$ eigenvectors of the graph Laplacian \cite{LiLiuTan2009}.
Closer to the original normalized graph cut problem are the approaches that start with the optimization problem of the spectral relaxation and 
add constraints that encode must-links and cannot-links \cite{YuShi2004, EriOlsKah2007, XuLiSch2009, WanDav2010}. 
Furthermore,  the case where the constraints are allowed to be inconsistent  is considered in  \cite{ColSauWir2008}.

In this paper we contribute in various ways to the area of graph-based constrained learning. First, we show in the spirit of 
$1$-spectral clustering \cite{HeiBue2010,HeiSet2011}, that the \emph{constrained} normalized cut problem has a \emph{tight}
relaxation as an \emph{unconstrained} continuous optimization problem. Our method, which we call COSC, is the first one in the field of
constrained spectral clustering, which can guarantee that all given constraints are fulfilled. While we present arguments that in practice
it is the best choice to satisfy all constraints even if the data is noisy,  in the case of inconsistent
or unreliable constraints one should refrain from doing so. Thus our second contribution is to show that our framework
can be extended to handle degree-of-belief and even inconsistent constraints. In this case COSC optimizes a trade-off between
having small normalized cut and a small number of violated constraints. We present an efficient implementation of COSC based
on an optimization technique proposed in \cite{HeiSet2011} which scales to large datasets. While the continuous optimization
problem is non-convex and thus convergence to the global optimum is not guaranteed, we can show that our method improves
any given partition which satisfies all constraints or it stops after one iteration.

All omitted proofs and additional experimental results can be found in the supplementary material.

{\bf Notation.}
Set functions are denoted by a hat, $\hat{S}$, while the corresponding extension 
is $S$. 
In this paper, we consider the normalized cut problem with general vertex weights.  
Formally, let $G(V, E, w, b)$ be an undirected graph $G$ with vertex set $V$ and edge set $E$ together
 with edge weights $w: V \times V \rightarrow \R_+$ 
and vertex weights $b: V \rightarrow \R_+$ and $n=|V|$.
Let $C \subset V$ and denote by $\overline{C} = V \backslash C$. 
We define respectively the cut, 
the generalized volume and the normalized cut (with general vertex weights) of a partition $(C,\overline{C})$ as
{\small
\begin{align*}
   \cut(C,\overline{C}) = 2\sum_{i \in C,\, j \in \overline{C}} w_{ij},&&
   \gvol(C)  = \sum_{i \in C} b_i,\\
   \bal(C)  = 2\frac{\gvol(C) \gvol(\overline{C})}{\gvol(V)} ,&&
      \Ncut(C,\overline{C}) = %\cut(C,\overline{C}) \Big(\frac{1}{\gvol(C)} + \frac{1}{\gvol(\overline{C*})}\Big) = 
                                             \frac{\cut(C,\overline{C})}{\bal(C)}.%\Big(\frac{1}{\gvol(C)} + \frac{1}{\gvol(\overline{C})}\Big).
%   \Ncut(C,\overline{C}) = \cut(C,\overline{C})\Big(\frac{1}{\gvol(C)} + \frac{1}{\gvol(\overline{C})}\Big).
\end{align*}}
We obtain ratio cut and normalized cut for special cases of the vertex weights,  $b_i = 1, 
 \ \textrm{and} \ b_i = d_i, \textrm{where} \ d_i = \sum_{j=1}^n w_{ij}$, respectively. In the ratio cut case, 
 $\gvol(C)$ is the cardinality of $C$ and in the normalized cut case, it is volume of $C$, denoted by $\vol(C)$.

\section{The Constrained Normalized Cut Problem}\label{sec:CLCon}
 We consider the normalized cut problem with must-link and cannot-link constraints. 
 Let $G(V, E, w, b)$ denote the given graph and $Q^m$, $Q^c$ be the constraint matrices, 
  where the element $q^m_{ij}$ (or $q^c_{ij}$) $\in \{0,1\}$ 
 specifies the must-link (or cannot-link) constraint between $i$ and $j$. We will in the following always assume that $G$ is connected. 
All what is stated below and our suggested algorithm can be easily generalized to \emph{degree of belief constraints} by allowing $q^{m}_{ij}$ (and $q^{c}_{ij}$)  $\in [0,1]$.
However, in the following we consider only $q_{ij}$ (and $q^{c}_{ij}$)  $\in \{0,1\}$,  in order to keep the theoretical statements more accessible.
\begin{definition}
  We call a partition $(C, \overline{C})$ \textbf{consistent} if it satisfies all constraints in $Q^m$  and $Q^c$.
\end{definition}
  Then the \textbf{constrained normalized cut problem} is to minimize $\Bcut(C, \overline{C})$ over all consistent
partitions. If the constraints are unreliable or inconsistent one can relax this problem and optimize 
a trade-off between normalized cut and the number of violated constraints.
 In this paper, we  address both problems in a common framework. 

%   finding a partition that violates at most given number of constraints.
 
We define the set functions, $M,N:2^V \rightarrow \R$, as
   \begin{align*}
   \hat{M}(C) &:= %\cut_{Q^m}(C, \overline{C}) = 
                           2 \sum_{i \in C,\, j \in \overline{C}} \hspace{-3mm}q^m_{ij}  \\%\ \textrm{and} \\
 \hat{N}(C) &:= %\vol(Q^c)  -  \cut_{Q^c}(C, \overline{C})=
                            \sum_{i \in C,\, j \in C} \hspace{-3mm} q^c_{ij}  +   \sum_{i \in \overline{C},\, j \in \overline{C}}  \hspace{-3mm}q^c_{ij}
                        = \vol(Q^c)  - 2 \sum_{i \in C,\, j \in \overline{C}}  \hspace{-3mm} q^c_{ij}.
 \end{align*}
$\hat{M}(C)$ and $\hat{N}(C)$  are equal to twice the number of 
 violated must-link and cannot-link constraints of partition $(C,\overline{C})$. 

As we show in the following, both the constrained normalized cut problem and its
soft version can be addressed by minimization of $\hat{F}_{\gamma} : 2^{V} \rightarrow \R$ defined as  
\begin{align}\label{set_func}
 \hat{F}_\gamma (C) 
=\frac{ \cut(C, \overline{C}) + \gamma \,\big(\hat{M}(C)
  +  \hat{N}(C)\big)}
% Qvol(\max(f) - \min(f)) - \gamma \ \sum_{i,j=1}^n q^c_{ij} \abs{f_i - f_j}}
{\bal(C)} ,
\end{align}
where $\gamma \in \R_+$. Note that $\hat{F}_\gamma(C)=\Ncut(C,\overline{C})$ if $(C,\overline{C})$ is consistent.
Thus the minimization of $\hat{F}_\gamma(C)$ corresponds to a trade-off between having small normalized
cut and satisfying all constraints parameterized by $\gamma$.

The relation between the parameter $\gamma$ and the number of violated constraints by the 
partition minimizing $\hat{F}_\gamma$ is quantified in the following lemma.
\begin{lemma} \label{le:gamma_bound}
 Let $(C,\overline{C})$ be consistent %with $Q^m$ and $Q^c$  
and $\lambda = \Ncut(C, \overline{C})$.
 If $\gamma \geq \frac{\gvol(V)}{4(l+1)} \lambda$, then
 any minimizer of $\hat{F_\gamma}$ violates no more than $l$ constraints.
\end{lemma}
 \begin{proof}
Any partition $(C,\overline{C})$ that violates more than $l$ constraints satisfies
$\hat{M}(C)+\hat{N}(C)\geq 2(l+1)$ and thus
{\small\begin{align*}
   \hat{F}_\gamma(C) &=   \Ncut(C, \overline{C})+ \gamma \ \frac{ \hat{M}(C) + \hat{N}(C) } {\bal(C)}\\
     & \ge \Ncut(C, \overline{C}) +  \frac{ \gamma 2(l+1)}{\frac{1}{2}\gvol(V)}
       > \frac{4 \gamma (l+1)}{\gvol(V)},
\end{align*}}where we have used that $\max_C \bal(C)=\gvol(V)/2$ and $\Ncut(C,\overline{C})>0$ as the graph
is connected. Assume now that  the partition $(D,\overline{D})$ minimizes $\hat{F}_\gamma$ for 
  $\gamma=\frac{\gvol(V)}{4(l+1)}\lambda$ and  violates more than $l$ constraints. Then
%  \[
{\small\begin{align*}
   \hat{F}_{\gamma}(D) &> \frac{4 \gamma (l+1)}{\gvol(V)} \geq \lambda = \hat{F_\gamma}(C),
%   \]
\end{align*}}
which leads to a contradiction.
 \end{proof}

Note that it is easy to construct a partition which is consistent and thus the above choice of $\gamma$
is constructive.
The following theorem is immediate from the above lemma for the special case $l=0$. 
\begin{theorem}\label{th:main_set}
 Let $(C,\overline{C})$ be consistent with the given constraints 
 and $\lambda = \Ncut(C, \overline{C})$.
Then for  $\gamma \geq \frac{\gvol(V)}{4} \ \lambda$, it holds that 
\begin{align*}
%\min_{f \in R^n} \frac{R(f) +  \gamma \sum_{i, j=1}^n q_{ij} \abs{\max(f) - \min(f)} - \gamma \sum_{i, j=1}^n q_{ij} \abs{f_i - f_j}}{S(f)} = \min_{\substack{C \subset V\\ \ \subj Q}} \Bcut(C, \overline{C})
\argmin_{\substack{C \subset V : \\ (C, \overline{C}) \ \textrm{consistent}}} \Bcut(C, \overline{C}) 
		= \argmin_{C \subset V} \hat{F}_\gamma(C)
\end{align*}
and the optimum values of both problems are equal. 
\end{theorem}
Thus the constrained normalized cut problem can be equivalently formulated as the 
combinatorial problem of minimizing $\hat{F}_\gamma$. In the next section we will show that this
problem allows for a tight relaxation into a continuous optimization problem.

\subsection{A tight continuous relaxation of $\hat{F}_\gamma$}
 Minimizing $\hat{F}_\gamma$  is a hard combinatorial problem. 
 In the following, we derive an equivalent continuous optimization problem. 
 Let  $f : \R^V \rightarrow \R$ denote a function on $V$, and $\ones_C$ denote the vector that is 1 on $C$ and 0 elsewhere.
  Define
%  \[ 
  \begin{align*}
   M(f) &:= \sum_{i,j=1}^n q^m_{ij} \abs{f_i - f_j} \textrm{and} \\
 N(f) &:= \vol(Q^c)\big(\max(f) - \min(f)\big) - \hspace{-1mm} \sum_{i,j=1}^n q^c_{ij} \abs{f_i - f_j},
 \end{align*}
% \]
where $\max(f)$ and $\min(f)$ are respectively the maximum and minimum elements of $f$. 
Note that $M(\ones_C)=\hat{M}(C)$ and $N(\ones_C)=\hat{N}(C)$ for any non-trivial\footnote{
A partition $(C, \overline{C})$ is non-trivial if neither $C=\emptyset$ nor ${C}=V$.} partition $(C,\overline{C})$. 
%  Note again that $M(\ones_C)$ and $N(\ones_C)$ give the number of must-link and 
%  cannot-link constraints violated by the partition $(C, \overline{C})$. 

Let $B$ denote the diagonal matrix with the vertex weights $b$ on the diagonal. We define % and $\gamma (\ge 0)$ is 

%  Since we also allow constraints to be inconsistent, we consider the following 
%  regularized functional for encoding must-link and cannot-link constraints.   
\begin{align*}
F_\gamma (f) %= \frac{R_\gamma(f)}{S(f)}
 = \frac{ \sum_{i,j = 1}^n w_{ij} \abs{f_i - f_j} + \gamma \ M(f)
  + \gamma \ 
  N(f)}
% Qvol(\max(f) - \min(f)) - \gamma \ \sum_{i,j=1}^n q^c_{ij} \abs{f_i - f_j}}
{{\norm{B (f - \frac{1}{\gvol(V)} \inner{f,b} \ones )}_1}}.
\end{align*}
We denote the numerator of $F_\gamma(f)$ by $R_\gamma(f)$ and the denominator
by $S(f)$.
\begin{lemma}\label{le:FGammaOnIndicators}
For any non-trivial partition it holds that $\hat{F}_\gamma(C) = F_\gamma(\ones_C)$.
\end{lemma}
\begin{proof}
We have,
\begin{align*} 
   &\sum_{i,j=1}^n w_{ij} | (\ones_C)_i - (\ones_C)_j| =  \cut(C,\overline{C})\\
  &\norm{B (\ones_C - \frac{1}{\gvol(V)} \inner{\ones_C,b} \ones )}_1 \\
= &\norm{B\Big (\big(1-\frac{\gvol(C)}{\gvol(V)}\big)\ones_C - \frac{\gvol(C)}{\gvol(V)} \ones_{\overline{C}}\Big)}_1\\
= &\big(1-\frac{\gvol(C)}{\gvol(V)}\big)\gvol(C) + \frac{\gvol(C)}{\gvol(V)}\gvol(\overline{C}) = \bal(C)
\end{align*}
This together with the discussion on $M,N$ finishes the proof.
\end{proof}
From Lemma \ref{le:FGammaOnIndicators} it immediately follows that minimizing $F_\gamma$
is a relaxation of minimizing $\hat{F}_\gamma$. In our main result (Theorem \ref{th:equ}), we establish that both problems are
actually equivalent, so that we have a tight relaxation. In particular a minimizer of $F_\gamma$ is an indicator function corresponding 
to the optimal partition minimizing $\hat{F}_\gamma$. 

The proof is based on the following key property of the functional $F_{\gamma}$. 
 Given any non-constant $f \in \R^n$, optimal thresholding,  
 \[ C^{*}_{f} = \argmin_{\min_i f_i\ \leq\ t\ <\ \max_i f_i} \hat{F}_\gamma(C^t_{f}),\] 
 where  $C^t_{f} = \{ i \in V | f_i > t\}$,  yields an indicator function on some $C^{*}_{f} \subset V$ with 
 smaller or equal value of $F_\gamma$. 
% Before proving the main result, we show a key property of the functional, $F_\gamma$. 
% Given any non-constant $f \in \R^n$, optimal thresholding yields an indicator function on some $C \subset V$ with 
% smaller or equal value of $F_\gamma$. We also give conditions under which optimal thresholding 
% yields a consistent partition. 
%\begin{lemma}\label{le:opt_thres}
% Let $f \in \R^n$ be non-constant and $(C^*_{f}, \overline{C^*_{f}})$ be the non-trivial partition obtained by 
%optimal thresholding of $f$,  that is 
%\[ C^*_{f} = \argmin_{\min_i f_i \leq t < \max_i f_i} \hat{F}_\gamma(C^t_{f}),\] 
%where $C^t_{f} = \{ i \in V | f_i > t\}$.
% Then 
% \[  F_\gamma(\ones_{C^*_{f}}) \le F_\gamma(f).\]
% Moreover, if $F_\gamma(f) \le \Bcut(D, \overline{D})$, for some consistent $(D, \overline{D})$ 
%  and $\gamma \geq \frac{\gvol(V)}{4} \ \Bcut(D, \overline{D})$, then 
% $C^*_{f}$ is also consistent.%with $Q$ and  $F_\gamma(\ones_{C^*_{f}}) = \Bcut(C^*_{f}, \overline{C^*_{f}})$.
%\end{lemma}

 \begin{theorem}\label{th:equ}
 For $\gamma \ge 0$, we have 
 \[
  \min_{C \subset V} \hat{F}_\gamma(C) \ = \   \min_{f \in \R^n,\ f \textrm{ non-constant}} F_\gamma(f). \]
  Moreover, a solution of the first problem can be obtained from the solution of the second problem. 
 \end{theorem}
\begin{proof}
%Let $C^t_f = \{ i \in V | f_i > t\}$. 
It has been shown in \cite{HeiBue2010}, that
{\small 
\begin{align*}
%\sum_{i,j = 1}^n w_{ij} \abs{f_i - f_j} 
\sum_{i,j=1}^n w_{ij}|f_i-f_j|   %&= 2 \ \sum_{f_i > f_j} w_{ij} (f_i - f_j) \\
%&= 2 \ \sum_{f_i > f_j} w_{ij} \int_{f_j}^{f_i} 1\, dt 
 %= 2 \ \int_{-\infty}^{\infty} \sum_{f_i > f_j} w_{ij} dt\\
 &= \ \int_{-\infty}^{\infty} \cut(C^t_f, \overline{C^t_f}) dt
\end{align*}
}
We define $\hat{P}:2^V \rightarrow \R$ as $\hat{P}(C) = 1$, if  $C\neq V$ and $C\neq \emptyset$, and $0$ otherwise. 
 Denoting by $\cut_{Q}(C, \overline{C})$, the cut on the constraint graph whose weight matrix is given by $Q$, we have
{\small 
\begin{align*}
R_\gamma(f) &= \ \int_{-\infty}^{\infty}  \cut(C^t_f, \overline{C^t_f}) dt  +\gamma\ \int_{-\infty}^{\infty}  \cut_{Q^{m}}(C^t_f, \overline{C^t_f}) \\  
&+ \gamma \vol(Q^{c}) \int_{\min_i f_i}^{\max_i f_i}1 dt
-  \gamma\ \int_{-\infty}^{\infty}  \cut_{Q^{c}}(C^t_f, \overline{C^t_f}) \\ 
 &=  \ \int_{-\infty}^{\infty}  \cut(C^t_f, \overline{C^t_f}) dt +\gamma\ \int_{-\infty}^{\infty}  \cut_{Q^{m}}(C^t_f, \overline{C^t_f}) \\  
& +\gamma \vol(Q^{c}) \int_{-\infty}^{\infty} \hat{P}(C^t_f) dt  - \gamma\ \int_{-\infty}^{\infty}  \cut_Q^{c}(C^t_f, \overline{C^t_f}) \\
%& = \int_{-\infty}^{\infty} \big(2 \ \cut(C^t_f, \overline{C^t_f}) - 2\gamma \ \cut_Q(C^t_f, \overline{C^t_f})\\
%& + \gamma \vol(Q) \hat{P}(C^t_f)\big) \ dt
 &= \int_{-\infty}^{\infty} \hat{R}_\gamma(C^t_f)\, dt
\end{align*}
}
Note that $S(f)$ is an even, convex and positively one-homogeneous function.\footnote{A function $S:\R^V \rightarrow \R$ 
is positively one-homogeneous if $S(\alpha f)=\alpha \,S(f)$ for all $\alpha >0$.}
Moreover, every even, convex positively one-homogeneous function, $T: \R^V \rightarrow \R$ has the form 
{\small 
$T(f) = \sup_{u \in U} \inner{u,f}$,
}where $U$ is a symmetric convex set, see e.g., \cite{HirLem2001}. Note that $S(\ones)=0$ and thus because of the symmetry of $U$ it has
to hold $\inner{u,\ones}=0$ for all $u \in U$. 
Since $S(\ones_{C^{t}_{f}}) = \hat{S}(C^{t}_{f})$ and 
$\innerproduct{u}{f} \le S(f), u \in U$, we have for all $u \in U$, %where $U \subset R^n$ is a closed convex set satisfying $\innerproduct{u}{b} = 0$,
{\small %\frac{\hat{R}_\gamma(C^t_f)}{\hat{S}(C^t_f)}
\begin{align}\label{lb_R_gamma}
%R_\gamma(f) &\ge \int_{-\infty}^{\infty} \hat{F}_\gamma(C^t_f)  \innerproduct{u} {\ones_{C_t}} dt 
R_\gamma(f) &\ge \int_{-\infty}^{\infty} \frac{\hat{R}_\gamma(C^t_f)}{\hat{S}(C^t_f)}  \innerproduct{u} {\ones_{C^t_f}} dt \nonumber\\
& \ge \inf_{t \in \R} \frac{\hat{R}_\gamma(C^t_f)}{\hat{S}(C^t_f)} \int_{\min_i f_i}^{\max_i f_i}  \innerproduct{u} {\ones_{C^{t}_{f}}} dt,
\end{align}
}where in the last inequality we changed the limits of integration using the fact that $\inner{u, \ones} = 0$.
Let $C_i:= C^t_{f_i}$ and $C_0 = V$. Then
{\small
\begin{align*}
\int_{\min_i f_i}^{\max_i f_i} \inner{u, \ones_{C_i}} dt = 
\sum_{i=1}^{n-1} \inner{u, \ones_{C_i}} (f_{i+1} - f_i) = \\
\sum_{i=1}^{n} f_i \big(\inner{u, \ones_{C_{i-1}}} - \inner{u, \ones_{C_i}} \big) = \sum_{i=1}^{n} f_i u_i =  \innerproduct{u}{f}
\end{align*}
}Noting that (\ref{lb_R_gamma}) holds for all $u \in U$, we have 
{\small
\begin{align*}
R_\gamma(f)% & \ge  \min_{t \in R} \hat{F}_\gamma(C^t_f) \innerproduct{u}{f}
\ge \inf_{t \in \R} \hat{F}_\gamma(C^t_f)\ \sup_{u \in U} \innerproduct{u}{f} %\\ & \ \mathrm{(since \ this \ holds \ for \ any \ u \in U)}\\
 = \inf_{t \in \R} \hat{F}_\gamma(C^t_f)\ S(f).
\end{align*}
}This implies that 
{\small
\begin{align}\label{thresholding_better_result}
F_\gamma(f) 
\ge \inf_{t \in \R} \hat{F}_\gamma(C^t_f) % = \min_{t \in R} \hat{F}_\gamma(C_f^t) 
= F_\gamma(\ones_{C^*_{f}}),
\end{align}
%= \hat{F}_\gamma(C^*_{f})$.
}
where $C^*_{f} = \argmin_{\min_{i} f_{i}\ \le\ t\ <\ \max_{i} f_{i}} \hat{F}_\gamma(C^t_{f})$. 

 This shows that we always get descent by optimal thresholding. 
 Thus the actual minimizer of $F_\gamma$ is a two-valued function, 
 which  can be transformed to an indicator function on some $C \subset V$,
 because of the scale and shift invariance of $F_{\gamma}$.
 Then from Lemma \ref{le:FGammaOnIndicators}, which shows that 
 for non-trivial partitions, $\hat{F}_{\gamma}(C) = F_{\gamma}(\ones_{C})$, 
  the statement follows. 
% corresponding to a partition minimizing $\hat{F}_\gamma$.
\end{proof}

 Now, we state our second result: the problem of minimizing the functional $F_\gamma$ 
 over arbitrary real-valued non-constant $f$, for a particular choice of $\gamma$, is in fact equivalent
 to the NP-hard problem of minimizing normalized cut with constraints. 

\begin{theorem}\label{th:main}
Let $(C, \overline{C})$ be consistent and $\lambda = \Bcut(C^\prime, \overline{C^\prime})$. 
Then for  $\gamma \ge \frac{\gvol(V)}{4} \ \lambda$, it holds that 
\begin{align*}
%\min_{f \in R^n} \frac{R(f) +  \gamma \sum_{i, j=1}^n q_{ij} \abs{\max(f) - \min(f)} - \gamma \sum_{i, j=1}^n q_{ij} \abs{f_i - f_j}}{S(f)} = \min_{\substack{C \subset V\\ \ \subj Q}} \Bcut(C, \overline{C})
\min_{\substack{C \subset V : \\ (C, \overline{C}) \ \textrm{consistent}}} \Bcut(C, \overline{C}) %= \min_{C \subset V} \hat{F}_\gamma(C)
=\min_{f \in \R^n,\ f \textrm{ non-constant}} F_\gamma(f) %\min_{f \in R^n} \frac{R_\gamma(f)}{S(f)}
\end{align*}
Furthermore, an optimal partition of the constrained problem can be obtained from a minimizer of the 
right problem. 
%Furthermore, if $f^*$ is any optimal solution of the last problem, then optimal thresholding of $f^*$ yields a consistent partition
% $(C^{*}_{f^{*}}, \overline{C^{*}_{f^{*}}})$ 
% which is also a minimizer of the first problem. 
 \begin{proof}
  From Theorem \ref{th:main_set} we know that, for the chosen value of $\gamma$, the 
  constrained problem is equivalent to 
  \[ \min_{C \subset V} \hat{F}_{\gamma}(C), \]
   which in turn is equivalent, by Theorem \ref{th:equ}, to the right problem in the statement. 
   Moreover, as shown in Theorem \ref{th:equ}, minimizer of $F_{\gamma}$ is an indicator 
   function on $C \subset V$ and hence we immediately get an optimal partition of the constrained problem. 
%  To show that optimal thresholding yields a consistent partition, note that for the given value of $\gamma$, we have 
%%\begin{align*}
%{\small \[\hat{F}_\gamma(C^*_{f^{*}})= F_\gamma(\ones_{C^*_{f^{*}}}) \le F_\gamma(f^{*}) \le \Ncut(C^{\prime}, \overline{C^{\prime}}) \le \frac{4\gamma}{\gvol(V)}.\]
%}
%However, if $(C^{*}_{f^{*}},\overline{C^{*}_{f^{*}}})$ violates at least one constraint then,
%{\small  $\hat{F}_\gamma(C^*_f) > \frac{4\gamma}{\gvol(V)}$},
%which would lead to a contradiction, thus the partition  $(C^{*}_{f^{*}},\overline{C^{*}_{f^{*}}})$ is 
%consistent.% and  $F_\gamma(\ones_{C^*_{f}}) = \Bcut(C^*_{f}, \overline{C^*_{f}})$. 
% \end{align*}
 \end{proof}
%Furthermore, for any $g \in \R^n$, such that $F_\gamma(g) \le \lambda$, let $(C^*_{g}, \overline{C^*_{g}})$ be the partition obtained by 
%optimal thresholding of $g$,  
%where $C^*_{g} = \argmin_{t} \hat{F}_\gamma(C^t_{g})$, 
%and for $t \in \R$, $C^t_{g} = \{ i \in V | g_i > t\}$.
% Then 
% $ F_\gamma(g) \ge F_\gamma(\ones_{C^*_{g}}) = \Bcut(C^*_{g}, \overline{C^*_{g}})$ and 
% $C^*_{g}$ is consistent with $Q$.
% \begin{proof}
% For the given value of $\gamma$ and for any partition $(C, \overline{C})$ that violates given constraints, 
%  we have, by Lemmas \ref{le:FuncCutLB} and \ref{le:FuncCutEq}, 
% $\hat{F}_\gamma(C) \ge  \frac{4 \gamma}{\gvol(V)} > \lambda = \hat{F}_\gamma(C^\prime)$. This implies any
%  inconsistent $C$ cannot be a minimizer of $\hat{F}_\gamma$. Thus we have
%  \[
% \ \min_{C \subset V} \hat{F}_\gamma(C) \ =  \min_{\substack{C \subset V \\ (C, \overline{C}) \ \textrm{consistent}}} \hat{F}_\gamma(C).
%% = \min_{\substack{C \subset V \\ C \ \textrm{consistent \ with Q}}} \Bcut(C, \overline{C}), \
%  \]

%The last equality follows from the Theorem \ref{th:equ}.
%Finally optimal thresholding, which does not increase the functional value,
% of any minimizer of $F_\gamma$ should result in a consistent partition. 

% %Proof can be found in the supplementary material.% \ref{app:main_result}.
% \end{proof}
\end{theorem}
A few comments on the implications of  Theorem \ref{th:main}. 
First, it shows that the constrained normalized cut problem can be equivalently
solved by minimizing $F_\gamma(f)$ for the given value of $\gamma$. 
 The value of $\gamma$ depends on the normalized cut value of a partition consistent with given constraints. 
 Note that such a partition can be obtained in polynomial time by 2-coloring the constraint graph as long as the constraints are 
 consistent. 

\subsection{Integration of must-link constraints via sparsification} \label{sec:ml_sparse}
If the must-link constraints are reliable and therefore should be enforced, one can directly integrate them
by merging the corresponding vertices together with re-definition of edge and vertex weights. 
In this way ones derives a new reduced graph, where the value of the normalized cut of all partitions that 
satisfy the must-link constraints are preserved. 
  
The construction of a reduced graph is given below for a must-link constraint $(p,q)$. 
\begin{enumerate}
\item merge $p$ and $q$ into a single vertex $\tau$.
\item update the vertex weight of $\tau$ by $b_\tau = b_p + b_q$.
\item update the edges as follows: if $r$ is any vertex other than $p$ and $q$, then add an edge between $\tau$ 
and $r$ with weight $w(p,r) + w(q,r)$.
\end{enumerate}

Note that this construction leads to a graph with vertex weights even if the original graph had vertex weights equal to $1$.
If there are many must-links, one can efficiently integrate all of them together by first constructing the must-link constraint graph and merging 
each connected component in this way.
 
The following lemma shows that the above construction preserves all normalized cuts which respect the must-link constraints.
 We prove it for the simple case where we merge $p$ and $q$ and the proof can easily be extended to the general case
by induction. 
\begin{lemma}\label{le:PreservationSCUT}
Let $G^\prime(V^\prime, E^\prime, w^\prime, b^\prime)$ be the reduced graph of $G(V, E, w, b)$ obtained by merging vertices 
$p$ and $q$. 
If a partition $(C,\overline{C})$ does not separate $p$ and $q$, we have
$\Bcut_G(C,\overline{C}) = \Bcut_{G^\prime}(C^\prime, \overline{C^\prime})$.
\end{lemma}
\begin{proof}
Note that $\gvol(V')=\gvol(\tau) + \gvol(V \backslash \{p,q\}) = \gvol(p) +\gvol(q) + \gvol(V \backslash \{p,q\}) = \gvol(V)$. 
If $(C,\overline{C})$ does not separate $p$ and $q$, then we have either $\tau \subset C$ or $\tau \subset \overline{C}$.
W.l.o.g. assume that $\tau \subset C$.
The corresponding partition of $V'$ is then $C'= \tau \cup (C \ \{p,q\})$ and $\overline{C'} = \overline{C}$. We get 
{\small
\begin{align*}
  \cut(C',\overline{C'}) &= \sum_{i \in C', j \in \overline{C'}} w'_{ij} = \sum_{j \in \overline{C'}} w'_{\tau j} + \sum_{i \in C'\backslash \tau, j \in \overline{C'}} w'_{ij} \\
                         &= \sum_{k \in \{p,q\}, j \in \overline{C}} w_{kj} + \sum_{i \in C\backslash \{p,q\}, j \in \overline{C}} w_{ij} = \cut(C,\overline{C}).\\
  \gvol(C')               &= \sum_{i \in C'} b'_i = b'_\tau + \sum_{i \in C'\backslash \tau}\\ 
                         &= b_p + b_q + \sum_{i \in C \backslash \{p,q\}} b_i = \sum_{i \in C} b_i = \gvol(C).\\
  \gvol(\overline{C'})    &= \sum_{i \in \overline{C'}} b'_i = \sum_{i \in \overline{C}} b_i = \gvol(\overline{C}).
\end{align*}
}
Thus we have $\Bcut_G(C,\overline{C}) = \Bcut_{G^\prime}(C^\prime, \overline{C^\prime})$
\end{proof}

\emph{All} partitions of the reduced graph fulfill all must-link constraints and thus any relaxation of the \emph{unconstrained} normalized cut problem can now
be used. Moreover, this is not restricted to the cut criterion we are using but any other graph cut criterion based on cut and the volume of the 
subsets will be preserved in the reduction.

%\begin{remark} In the spectral relaxation of the constrained normalized cut problem having only ML constraints, we can thus use solution of the
%unconstrained probem %\eqref{eq:UnconstrainedSC} 
%which is given by the second eigenvector of the generalized eigenproblem, \[ (D'-W')f = \lambda B'f,\]
%where $D',W',B'$ are the degree, edge weight and vertex weight matrices of the reduced graph. In this way we have recovered in an elegant manner the method of Yu
%and Shi \cite{YuShi2004} which integrates ML constraints in a more complicated process. Also our integration of ML constraints
%works on the graph level and thus can be used by every other method, whereas their derivation is restricted to the normalized cut problem.
%\end{remark}
%\begin{remark}
%If not explicitly stated otherwise we assume in the following to work with the reduced graph where the must-link constraints are integrated.
%\end{remark}

\section{Algorithm for Constrained $1$-Spectral Clustering}
In this section we discuss the efficient minimization of $F_\gamma$ based on recent ideas from unconstrained
$1$-spectral clustering \cite{HeiBue2010,HeiSet2011}. Note, that $F_\gamma$ is a non-negative ratio of a difference of convex (d.c) function 
and a convex function, both of which
 are positively one-homogeneous.
% \footnote{A function $S:\R^V \rightarrow \R$ 
%is one-homogeneous if $S(\alpha f)=\alpha S(f)$ for all $\alpha \in \R$.}
In recent work \cite{HeiBue2010,HeiSet2011}, a general scheme, shown in  Algorithm \ref{alg:ratio_dc} (where $\partial S(f)$ denotes the subdifferential of the convex function $S$ at $f$), is proposed for the 
minimization of a non-negative ratio of a d.c function and convex function both of which are positively one-homogeneous.
%The main idea of this scheme is to replace in each step the second component of the d.c. function by its
%affine minorization which results in a convex majorant and then solve this ratio of convex functions using the 
%scheme from \cite{HeiBue2010}. 

It is shown in \cite{HeiSet2011} that Algorithm \ref{alg:ratio_dc} generates a sequence $f^{k}$ such that either
$F_\gamma(f^{k+1}) < F_\gamma(f^k)$ or the sequence terminates. Moreover, the cluster points of $f^{k}$
correspond to critical points of $F_\gamma$. 
   The scheme is given in Algorithm \ref{alg:ratio_dc} for the problem 
  $ \min_{f \in \R^n} (R_1(f) - R_2(f)) / S(f)$, where
{\small
\begin{align*}
   R_1(f) &:= \frac{1}{2}  \sum_{i, j = 1}^n (w_{ij} + \gamma q^m_{ij}) |f_i - f_j| \\
   &+ \frac{\gamma}{2} \sum_{i, j=1}^n q^c_{ij} (\max(f) - \min(f))\\
   R_2(f) &:= \frac{1}{2} \sum_{i,j=1}^n q^c_{ij} |f_i - f_j|, \qquad \\ S(f) &:=  \frac{1}{2}  \norm{B (f - \frac{1}{\gvol(V)} \inner{f,b} \ones )}_1
\end{align*}
}
Note that $R_1,\,R_2$ are both convex functions and $F_\gamma(f)=(R_1(f)-R_2(f))/S(f)$.
%  use the technique developed in \cite{HeiBue2010}. 
% Algorithmically it amounts to solving a sequence of convex programs and convergence to a critical 
% point is shown. 

\begin{algorithm}[htb]
   \caption{Minimization of a ratio $(R_1(f)-R_2(f))/S(f)$ where $R_1,R_2,S$ are convex and positively one-homogeneous}
   \label{alg:ratio_dc}
\begin{algorithmic}[1]
   %\STATE {\bfseries Input:}  accuracy $\epsilon$
   \STATE {\bfseries Initialization:} $f^0 = \text{random}$ with $\norm{f^0} = 1$, $\lambda^0 =(R_1(f^0)-R_2(f^0))/S(f^0)$
   \REPEAT
%   \STATE $f^{k+1} = \argmin_{\norm{f}_2 \leq 1} \left\{ R_1(f) - \inner{f, r_2(f^k)} - \lambda^k \inner{f, s(f^k)}    \right\}$ \\
%   			\text{where\ } $r_2(f^k) \in \partial R_2(f^k)$, $s(f^k) \in \partial S(f^k)$
   \STATE $f^{k+1} = \argmin_{\norm{f}_2 \leq 1} \left\{ R_1(f) - \inner{f, r_2} - \lambda^k \inner{f, s}    \right\}$ \\
   			\text{where\ } $r_2 \in \partial R_2(f^k)$, $s \in \partial S(f^k)$
   %\STATE $f^{k+1} = g^{k+1}/S(g^{k+1})^{1/p}$
  % \STATE $f^{k+1} = g^{k+1}/\norm{g^{k+1}}_2$
   \STATE $\lambda^{k+1}= (R_1(f^{k+1}) - R_2(f^{k+1})/S(f^{k+1})$
   %\UNTIL {$\abs{\min_{\norm{f}_2^2\leq 1} \left\{ R(f) - \lambda^k \inner{f, s(f^k)}  \right\} }< \epsilon$}
	\UNTIL $\frac{\abs{\lambda^{k+1}-\lambda^k}}{\lambda^k}< \epsilon$
  \STATE {\bfseries Output:}  $\lambda^{k+1}$ and $f^{k+1}$.
%	\STATE {\bfseries Output:}  $\lambda^{k+1}$ and eigenvector $f^{k+1}$.
\end{algorithmic}
\end{algorithm}
Moreover, it is shown in \cite{HeiSet2011}, that if one wants to minimize $(R_1(f)-R_2(f))/S(f)$ only over non-constant functions, one has to ensure
that $\inner{r_2,\ones}=\inner{s,\ones}=0$. Note, that
\begin{align*}
\partial S(f)   & =  \frac{1}{2}(\Ones-\frac{1}{\gvol(V)} b \ones^T)B\sign\big(f-\frac{\innerproduct{b}{f}}{\gvol(V)}\ones\big)\\
\partial R_2(f) &= \Big\{\sum_{j=1}^n q^c_{ij} u_{ij}\,|\, u_{ij}=-u_{ji}, \, u_{ij} \in \mathrm{sign}(f_i-f_j)\Big\},
\end{align*}
where $\sign(x)=[-1,1]$ if $x=0$, otherwise it just the sign function.
It is easy to check that $\inner{u,\ones}=0$ for all $u \in \partial S(f)$ and all $f \in \R^n$ and there exists always a vector
$u \in \partial R_2(f)$ for all $f \in \R^n$ such that $\inner{u,\ones}=0$.

In the algorithm the key part is the inner convex problem which one has to solve at each step.
In our case it has the form, 
{\small
\begin{align}
\min_{\norm{f}_2 \le 1} &\;\frac{1}{2} \sum_{i, j = 1}^n (w_{ij}+ \gamma q^m_{ij})  \abs{f_i - f_j} \\ &+\frac{\gamma}{2} \sum_{i, j=1}^n q^c_{ij} (\max(f) - \min(f))
 -  \inner{f,\gamma\, r_{2} + \lambda^k s},\nonumber
\end{align}
where $r_{2} \in \partial R_2(f^{k})$, %\partial({\sum_{i,j=1}^n q_{ij} \abs{f_i - f_j}})$, 
$s \in \partial{S(f^{k})}$ and $\lambda^k=F_\gamma(f^k)$.
}

To solve it more efficiently we derive an equivalent smooth dual formulation for this non-smooth convex problem. 
 We replace $w_{ij}+ \gamma q^m_{ij}$ by $w^\prime_{ij}$ in the following.
\begin{lemma}
Let $E \subset V \times V$ denote the set of edges and $A : \R^E \rightarrow \R^V$ be defined as $(A\alpha)_i = 
\sum_{j | (i,j) \in E} w^\prime_{ij} \alpha_{ij}$. 
Moreover, let $U$ denote the simplex, $U = \{u \in \R^n \ | \ \sum_{i=1}^n u_i = 1, u_i \ge 0, \ \forall i \}$. 
The above inner problem is equivalent to 
\begin{align}\label{opt:org}
& \min_{\big\{\substack{\alpha \in \R^E |  \norm{\alpha}_\infty \le 1, \alpha_{ij} = - \alpha_{ji} \\ {v \in U}}\big\}} \Psi(\alpha, v) := \\
&  c \norm{-A\frac{\alpha}{c} + v + b - P_U\big(-A\frac{\alpha}{c} + v + b\big)}_2,\nonumber
\end{align}
  where $c = \frac{\gamma}{2}{\vol(Q^c)}$, $b = \frac{r_{2}}{c} + \lambda^k \frac{s}{c}$ and $P_U(x)$ is the projection of $x$ on to the simplex $U$. 
    \begin{proof}
 Noting that {\small $\frac{1}{2}\sum_{i,j = 1}^n w^\prime_{ij} \abs{f_i - f_j} = 
 \max_{\{\alpha \in \R^E |  \norm{\alpha}_\infty \le 1, \alpha_{ij} = - \alpha_{ji}\}} \inner{f, A\alpha}$} (see \cite{HeiBue2010}) and
 {\small $\max_{i} f_i = \max_{u \in U} \inner{u, f}$},
 the inner problem can be rewritten as 
 \begin{align*}
 &\min_{\norm{f}_2 \le 1} \max_{\{\alpha \in \R^E |  \norm{\alpha}_\infty \le 1, \alpha_{ij} = - \alpha_{ji}\}} \inner{f, A\alpha} \\
& + c \max_{u \in U} \inner{f, u} + c \max_{v \in U} \inner{-f, v}  - \gamma \inner{f, r_{2}}  - \lambda^k \inner{f, s} \\
& = \min_{\norm{f}_2 \le 1} \max_{\big\{\substack{{\alpha \in \R^E |  \norm{\alpha}_\infty \le 1, \alpha_{ij} = - \alpha_{ji}} \\
 { u, v \in U}}\big\}}\inner{f, A\alpha} \\
 &+ c \inner{f,u} - c \inner{f,v} - \gamma \inner{f, r_{2}} - \lambda^k \inner{f,s} \\
& {\stackrel {(s_1)}{=}} \max_{\substack{{\alpha \in \R^E |}\\{  \norm{\alpha}_\infty \le 1} \\{\alpha_{ij} = - \alpha_{ji}} \\
 { u, v \in U}}} \min_{\norm{f}_2 \le 1}  \inner{f, A\alpha + c(u - v)  - \gamma r_{2} - \lambda^k s} \\
%& {\stackrel {(s_1)}{=}} \max_{\substack{{\alpha \in \R^E |  \norm{\alpha}_\infty \le 1} {\alpha_{ij} = - \alpha_{ji}} 
% { u, v \in U}}} \min_{\norm{f}_2 \le 1}  \inner{f, A\alpha + cu - cv  - \gamma \inner{f, r} - \lambda^k s(f^k)} \\
 & {\stackrel {(s_2)}{=}} \max_{\substack{{\alpha \in \R^E |  \norm{\alpha}_\infty \le 1}\\ {\alpha_{ij} = - \alpha_{ji}} 
 { u, v \in U}}} -  \norm{A\alpha - \gamma r_{2} - \lambda^k s + c(u - v)}_2\\
&  = - \min_{\substack{{\alpha \in \R^E |  \norm{\alpha}_\infty \le 1}\\ {\alpha_{ij} = - \alpha_{ji}} \\
 { u, v \in U}}}  \Psi(\alpha, u, v). % \norm{A\alpha - \gamma r - \lambda^k s + cu - cv}_2
\end{align*}
 The step $s_1$ follows from the standard min-max theorem (see Corollary 37.3.2 in \cite{Roc70}) since $u$, $v$, $\alpha$ 
 and $f$ lie in non-empty compact convex sets.
 In the step $s_2$, we used that the minimizer of the linear function over the Euclidean ball is given by
 \begin{align*}
  f^* = \frac{A\alpha - \gamma r_{2} - \lambda^k s + cu - cv}{\norm{A\alpha - \gamma r_{2} - \lambda^k s + cu - cv}_2},
 \end{align*}
%  In the last equality we used the fact that the minimizer of linear function over the 
% Euclidean ball is given by the unit vector pointing in the negative cost direction. 
 if $\norm{A\alpha - \gamma r - \lambda^k s + cu - cv}_2 \ne 0$; otherwise $f^*$ is an arbitrary element of the 
 Euclidean unit ball. 
 
 Finally,  we have \norm{A\alpha - r - \lambda^k s + c u - cv} = c \norm{A\frac{\alpha}{c} + u - v - b}. 
  We also know that for a convex set $C$ and any given $y$, $\min_{x \in C} \norm{x-y} = \norm{y - P_C(y)}$, where 
  $P_C(y)$ is the projection of $y$ onto the set $C$. With $y = -A\frac{\alpha}{c} + v + b$, we have for any $\alpha$, 
  $\min_{u, v \in U} \Psi(\alpha, u, v)  = \min_{v \in U} \min_{u \in U} c \norm{u - y} =  \min_{v \in U} c \norm{y - P_U(y)}$
  and from this the result follows.
  \end{proof}
\end{lemma}

The smooth dual problem can be solved efficiently using first order projected gradient methods like FISTA \cite{BT09}, 
which has a guaranteed convergence rate of $O(\frac{L}{k^2})$, where $k$ is the number of steps,
and $L$ is the Lipschitz constant of the gradient of the objective.
% FISTA requires a bound on the Lipschitz constant of the gradient of the objective, to determine the step size.
% Due to ill-conditioning (that comes from the highly varying weight term 
% $w^\prime_{ij}$) 
%  of the above dual problem, the bound one gets on Lipschitz constant is weak and hence leads to
%  slower convergence. Hence we do the following change of variables which essentially transforms the problem to 
%  unweighted graph and improves the bound on Lipschitz constant. 
% In practice, we got significant improvements in the convergence speed because of this transformation. 
The bound on the Lipschitz constant for the gradient of the objective in \eqref{opt:org} can be rather loose if the weights are varying
a lot. The rescaling of the variable $\alpha$ introduced in Lemma \ref{le:optnew} 
leads to a better condition number and also to a tighter bound on the Lipschitz constant. 
This results in a significant improvement in practical performance.
%The variable transformation introduced in Lemma \ref{le:optnew} 
%leads to a tighter bound on the Lipschitz constant and thereby results in a significant 
%improvement in practical performance. 
%leads to a Lipschitz constant which is basically
%equivalent to the one for the unweighted graph. This leads to a significant improvement in practical performance.
\begin{lemma}\label{le:optnew}
%Let $B$ be a linear operator defined as $(B\beta)_i := \sum_{j : (i,j) \in E} \beta_{ij}$ and  
%let $ s_{ij} = \frac{w^\prime_{ij}} {c}{\norm{B}} $. The above inner problem is equivalent to 
Let $B$ be a linear operator defined as $(B\beta)_i := \sum_{j : (i,j) \in E} \beta_{ij}$ and  
let $ s_{ij} = \frac{w^\prime_{ij}} {c}{M}$, for positive constant $M \ge \norm{B}$. The above inner problem is equivalent to 
\[
\min_{\substack{\{\beta \in \R^E | \norm{\beta}_\infty \le s_{ij}, \beta{ij} = - \beta_{ji}\} \\ {v \in U}}} \tilde{\Psi}(\beta, v) := 
% \norm{-B\beta + r/C + \lambda^k s/c + v - P_U(-B\beta + r/C + \lambda^k s/c + v)}_2,\]
% c \norm{(-\frac{B}{\norm{B}}\beta + v + b) - P_U(-\frac{B}{\norm{B}}\beta + v + b)}_2,\]
 \frac{1}{2}\norm{d - P_U(d)}_2^{2},\]
 where $d = -\frac{B}{M}\beta + v + b$.
% where $b = r/c + \lambda^k s/c$ and $P_U(x)$ is projection of $x$ on to the simplex $U$. 
 The Lipschitz constant of the gradient of $\tilde{\Psi}$ is upper bounded by 4.
% $ \max\biggr\{ \frac{\norm{B^T}}{M}, 1\biggr\} \sqrt{8\ \biggr(1+ \frac{\norm{B^T}^2}{M^2}\biggr)}$.
\end{lemma}
  \begin{proof}
 Let $\beta_{ij} = s_{ij} \alpha_{ij}$. Then $(A\frac{\alpha}{c})_i = \sum_{j : (i,j) \in E} w^\prime_{ij}\frac{\alpha_{ij} }{c} 
 = \sum_{j : (i,j) \in E} \frac{\beta_{ij}}{M} 
 = (\frac{B}{M}\beta)_i$ and constraints on $\alpha$  transform to $-s_{ij} \le \beta_{ij} \le s_{ij}$ and 
$ \beta_{ij}  =  -\beta_{ij}$. 
Since the mapping between $\alpha$ and $\beta$ is one-to-one, the transformation yields an equivalent problem
 (in the sense that minimizer of one problem can be easily derived from minimizer of the other problem).

Now we derive a bound on the Lipschitz  constant. 

The gradient of $\Psi$ at $x=(\beta, v)$ w.r.t $\beta$, and $v$ are given by
  \begin{align*}
  (\nabla{\Psi(x)})_\beta = - \frac{B^T}{M} (d - P_U(d)), \\
%    (\nabla{\Psi(x)})_u = 2c \ g \ \textrm{and} \ 
    (\nabla{\Psi(x)})_v =  (d - P_U(d)),
  \end{align*}
   where $B^T$ is the adjoint operator of $B$ given by $(B^Td)_{ij} = (d_i - d_j)$.
   
  Let $x^\prime = (\beta^\prime, v^\prime)$ denote any other point and $d^\prime = -\frac{B}{M}\beta^\prime + v^\prime + b$.
   then we have
\begin{align*}
 &\norm{ \nabla{\Psi(x)} - \nabla{\Psi(x^\prime)} }^2 \\
 &=  \frac{\norm{B^T}^2}{M^2} \norm{(d  -P_U(d)) - (d^\prime - P_U(d^\prime)} ^2 \\
  &+ \norm{(d-P_U(d)) - (d^\prime - P_U(d^\prime))}^2 \\
 &=  \biggr(1+ \frac{\norm{B^T}^2}{M^2}\biggr)\norm{((d  - d^\prime) + (-P_U(d)  + P_U(d^\prime))} ^2 \\
& \le 2\ \biggr(1+ \frac{\norm{B^T}^2}{M^2}\biggr)  (\norm{d - d^\prime}^2  + \norm{-P_U(d)  + P_U(d^\prime)}^2  )\\
& \le 4\ \biggr(1+ \frac{\norm{B^T}^2}{M^2}\biggr) \norm{d - d^\prime}^2\\
% = 4 (\norm{B^T}^2 +1) \norm{d - d^\prime}^2 \\
& = 4\ \biggr(1+ \frac{\norm{B^T}^2}{M^2}\biggr)  \norm{\frac{B}{M} (-\beta + \beta^\prime) + (v - v^\prime) }^2\\
&  \le  4 \biggr(1+ \frac{\norm{B^T}^2}{M^2}\biggr)  \max\biggr\{ \frac{\norm{B^T}^2}{M^2}, 1\biggr\} \\
&\ \norm{ (-\beta + \beta^\prime) + (v - v^\prime) }^2\\
%&  \le  8 \times \ 2 \ ( \norm{ (-\beta + \beta^\prime)}^2 + (\norm{v - v^\prime) }^2)\\
& \le  8\ \biggr(1+ \frac{\norm{B^T}^2}{M^2}\biggr) \max\biggr\{ \frac{\norm{B^T}^2}{M^2}, 1\biggr\}\\
&   \ ( \norm{ \beta - \beta^\prime}^2 + \norm{v - v^\prime) }^2)\\
& =  8\ \biggr(1+ \frac{\norm{B^T}^2}{M^2}\biggr) \max\biggr\{ \frac{\norm{B^T}^2}{M^2}, 1\biggr\}  \  \norm{ x - x^{\prime}}^2 
%&  \le 2 \sqrt{\norm{B^T}^2 + 2 c^2}  \max\{\norm{A^T}, c\} \sqrt{3} \big(\norm{ (\alpha - \alpha^\prime)} + \norm{(u - u^\prime)} + \norm{ (v - v^\prime)} \big)
\end{align*}
%\sqrt{( \norm{ (\beta - \beta^\prime)}^2 + (\norm{v - v^\prime) }^2)}
%Thus we have $\norm{ \nabla{\Psi(x)} - \nabla{\Psi(x^\prime)} } \le 4  \ \norm{ x - x^{\prime}}$ and hence 
%the Lipschitz constant is upper bounded by 4.
Hence the Lipschitz constant is upper bounded by 
$ \sqrt{8\ \biggr(1+ \frac{\norm{B^T}^2}{M^2}\biggr) \max\biggr\{ \frac{\norm{B^T}^2}{M^2}, 1\biggr\}} \le 4,$ 
since $M \ge \norm{B}$.
\end{proof}

We can choose $M$ by upper bounding $\norm{B}$ using
\[ \norm{B}^2 \le \max_r \sum_{(r,j) \in E} 1^2 = \max_r \neigh(r),\]
where $\neigh(r)$ is the number of neighbors of vertex $r$.
%The above transformation works for any other constant $M$ instead of $\norm{B}$ but for the bound of Lipschitz constant to be valid 
%we need $M \ge \norm{B}$.  We can bound $\norm{B}$ by
%\[ \norm{B}^2 \le \max_r \sum_{(i,j) \in E} 1^2 = \max_r \neigh(r),\]
%where $\neigh(r)$ is the number of neighbors of vertex $r$.

Despite the problem of minimizing $F_\gamma$ is non-convex and thus global convergence is not guaranteed, Algorithm \ref{alg:ratio_dc} has the following quality guarantee.
%given any partition that satisfies all constraints (e.g., result of any other constrained clustering method) 
%as the initialization, the algorithm either yields a better partition that satisfies all constraints or stops immediately.
% given any partition as the initialization, the algorithm either yields a partition with better value of $F_\gamma$ or stops 
% immediately. 
 \begin{theorem}
 Let $(C, \overline{C})$ be any partition and let $\lambda = \Ncut(C,\overline{C})$. 
 If one uses $\ones_{C}$ as the initialization of the 
 Algorithm \ref{alg:ratio_dc}, then the algorithm either terminates in one step or outputs an $f^1$ which yields
 %which after optimal thresholding as done in Theorem \ref{th:main} gives 
 a partition $(A, \overline{A})$ such that 
% Let $(C, \overline{C})$ be any partition such that $C$ is consistent with $Q$ and let $\lambda = \Ncut(C,\overline{C})$. 
% If one sets for $\gamma$ any value larger than $\frac{\gvol(V)}{8} \lambda$  and uses $\ones_{C}$ as the initialization of the 
% Algorithm \ref{AlgoCBCut}, then the algorithm either terminates in one step or outputs an $f^1$ which after optimal thresholding 
% as done in Theorem \ref{th:main} gives a partition $(A, \overline{A})$ such that $A$ is consistent with $Q$ and 
 \begin{align*}
  %\Ncut(A,\overline{A})  <  \Ncut(C,\overline{C}) 
  \hat{F}_\gamma(A)  <  \hat{F}_\gamma(C)
 \end{align*}
 Moreover, if $(C, \overline{C})$ is consistent and if we set for $\gamma$ any value larger than $\frac{\gvol(V)}{4} \lambda$ 
 then $A$ is also consistent and  $\Ncut(A,\overline{A})  <  \Ncut(C,\overline{C})$. 
 \end{theorem}
 \begin{proof}
  Algorithm \ref{alg:ratio_dc} generates $\{f^k\}$ such that 
 $ F_\gamma(f^{k+1})  < F_\gamma(f^k)$ until it terminates \cite{HeiSet2011}, 
 we have $F_\gamma(f^1) < F_\gamma(f^0) = \hat{F}_\gamma(C)$, if the algorithm does not stop in one step. 
  As shown in theorem \ref{th:equ}, optimal thresholding of $f^1$ results in a partition $(A, \overline{A})$ such that $\hat{F}_\gamma(A)
  \le F_\gamma(f^1) < \hat{F}_\gamma(C)$.
 If $C$ is consistent, we have  $F_\gamma(f^0)= \hat{F}_\gamma(C) =\Ncut(C,\overline{C}) = \lambda$, 
 by Lemma \ref{le:FGammaOnIndicators}.
 %This implies, together with Lemma \ref{le:FuncCutLB} for the chosen value of $\gamma$, that $A$ is consistent with $Q$. 
  For the chosen value of $\gamma$, using a singular argument as in  Lemma \ref{le:gamma_bound}, one sees that for any inconsistent subset $B$, 
   $\hat{F}_\gamma(B) > \lambda$ and hence $A$ is consistent with $Q$. 
  Then it is immediate that $ \Ncut(A, \overline{A}) = \hat{F}_\gamma(A)  \le F_\gamma(f^1) < F_\gamma(f^0) = \Ncut(C, \overline{C})$. 
  \end{proof}
 In practice,  the best results can be obtained by first minimizing $F_\gamma$ for 
$\gamma=0$ (unconstrained problem) and then increase the value of $\gamma$ and use the previously obtained
clustering as initialization. This process is iterated until the current partition violates not 
more than a given number of constraints.

 \section{Soft- versus Hard-Constrained Normalized Cut Problem}\label{sec:soft_ver}
 The need for a soft version arises, for example, if the constraints are noisy or inconsistent. 
   Moreover, as we illustrate in the next section, we use the soft version to extend
   our clustering method to the multi-partitioning problem.
% Since the constraints are encoded via a penalty term in \eqref{set_func}, we do not strictly enforce 
%the constraints  but optimize a trade-off between normalized cut and number of violated constraints
%is minimized. Moreover, u
Using the bound of Lemma \ref{le:gamma_bound} for $\gamma$, we can solve the soft constrained problem
   for any given number of violations.
%We prefer the latter approach in all our experiments since directly minimizing 
  % our functional for large $\gamma$ might result in sub-optimal solutions. 
   \begin{figure}
   \centering
   \includegraphics[width=60mm]{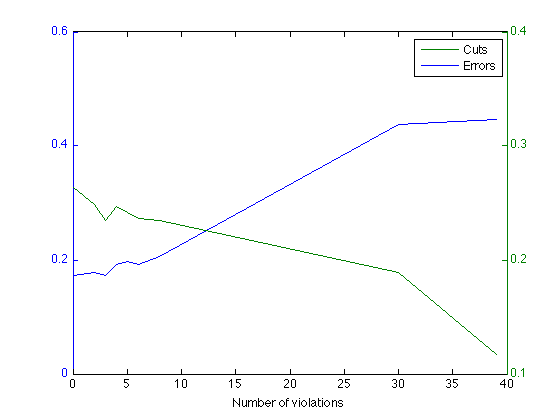}
\caption{\label{fig:dec_gamma} Influence of $\gamma$ on cut and clustering error.}
\end{figure}

It appears from a theoretical point of view that, due to noise, satisfying all constraints should not be the best choice. 
However, in our experiments it turned out, that typically the best results were achieved when all constraints were satisfied. 
We illustrate this behavior for the dataset Sonar, where we generated 80 constraints and increased $\gamma$ from zero
until all constraints were satisfied.  In Figure \ref{fig:dec_gamma}, we plot cuts and errors versus the number of violated constraints.  One observes that the best error is obtained when all constraints were satisfied.
Since by 
enforcing always all given constraints, our method becomes parameter-free (we increase $\gamma$ until all constraints are satisfied), we chose this option for the experiments. 
% It is not clear whether satisfying all constraints is helpful in practice. 
% To analyze this, we conducted an experiment on a real world dataset (of size 200) where we generated 80 constraints and ran 
% our method for different values of $\gamma$. 
% In figure \ref{fig:dec_gamma}, we plot cuts and errors versus the number of constraints violated.    
%  In the plots the rightmost point on the x-axis corresponds to the case $\gamma = 0$.
% This example suggests that satisfying constraints is helpful as the clustering error decreases. 
%% Here satisfying constraints seem to decrease the clustering error and 
% Hence in all our experiments we chose to satisfy all constraints which also makes our method parameter free. 
%   We illustrate this procedure of solving a sequence of soft constrained problems on one of the real world datasets 
%   obtained from UCI repository.
%   results of soft constrained problem on a real world data set obtained from UCI repository.
%  We ran this procedure on an artificial dataset containing 600 points in 100 dimensions. 
%   In figure \ref{fig:dec_gamma}, we plot cuts and errors versus the number of constraints violated.    
%  
%:
 \section{Multi-Partitioning with Constraints}
%   One of the standard ways to compute multi partitioning is to use recursive bi-partitioning approach.
%   In this section we present a method for computt
   In this section we present a method to integrate constraints in a multi-partitioning setting. 
   In the multi-partitioning problem, one seeks a $k$-partitioning $(C_1, \ldots, C_k)$ 
   of the graph such that the normalized multi-cut given by
 \begin{equation}\label{eq:multicut}
 \sum_{i=1}^k \Bcut(C_i, \overline{C_i})
\end{equation}
   is minimized.
A straightforward way to generate a multi-partitioning is to use a
recursive bi-partitioning scheme. Starting with all points as the
initial partition, the method repeats the following steps until
the current partition has $k$ components.

\begin{enumerate}
\item split each of the components in the current partition into two parts.
\item  choose among the above splits the one minimizing the multi-cut criterion.
\end{enumerate}

Now we extend this method to the constrained case. 
Note that it is always possible to perform a binary split which satisfies
all must-link constraints. Thus, must-link constraints pose no difficulty
in the multi-partitioning scheme, as all must-link constraints can be
integrated using the procedure given in \ref{sec:ml_sparse}. 

However, satisfying all cannot-link constraints is sometimes not possible
(cyclic constraints) and usually also not desirable at each level of
the recursive bi-partition, since an early binary split cannot
separate all classes.
The issues here is which cannot-link constraints should be considered
for the binary split in step 1.
% and how the criteria in step 2 should
%be extended to incorporate the number of violated cannot-link constraints.

To address this issue, we use the soft-version of our formulation
where we need only to specify the maximum number, $l$, of violations
allowed. We derive this number $l$ assuming the following simple uniform
model of the data and constraints. We assume that all classes have equal
size and there is an equal number of cannot link constraints between all
pairs of classes. Assuming that any binary split  does not destroy
the class structure, the maximum number of violation is obtained if
one class is separated from the rest. Precisely, the
expected value of this number, given $N$ cannot-link constraints and $k$
classes, is $\frac{(k-1)(k-2)/2}{k (k-1)/2}\ N$. 
%Analogously one can derive
%the maximum number of possible violations for all
%successive splits.
In the first binary split, these numbers ($N$ and $k$) are known. In the succesive binary splits, $N$ is known, 
while $k$ can again be derived, assuming the uniform model, as $\frac{k}{n}\ \tilde{n}$, where $\tilde{n}$ is 
the size of the current component. 

     We illustrate our approach using an artificial dataset (mixture of Gaussians, 500 points, 2 dimensions).
%that is modeled after a typical scenario that happens in USPS  digit dataset. 
Figure \ref{fig:multi_toy} shows on the left the ground truth and the solution of unconstrained ($\gamma$=0)
     multi-partitioning. In the unconstrained solution, points belonging to the same class are split into two clusters while points from other two classes are merged into a single cluster. On the rightmost, the result of our constrained multi-partitioning framework 
     with 80 randomly generated constraints is shown. 
     
     \begin{figure}
     \begin{tabular}{ccc}
\hspace{-9mm}
 \includegraphics[width=0.2\textwidth]{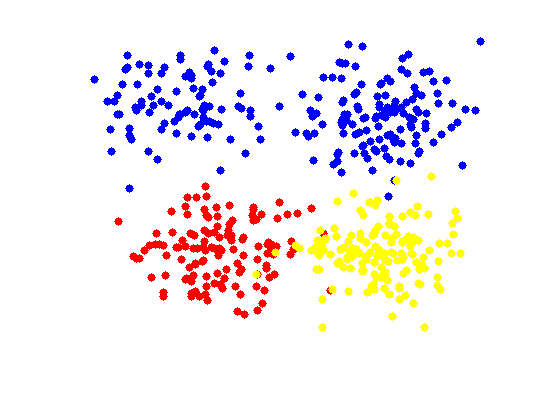}
\hspace{-7mm}
   \includegraphics[width=0.2\textwidth]{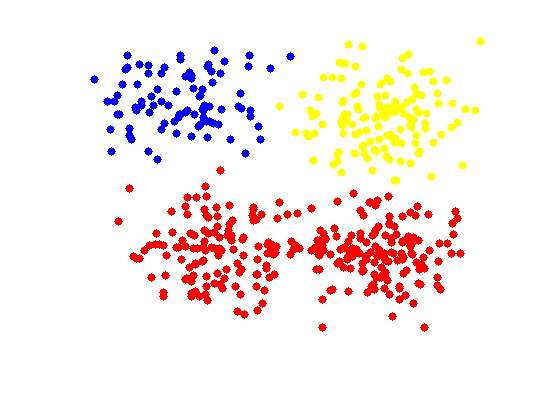}
\hspace{-7mm}
     \includegraphics[width=0.2\textwidth]{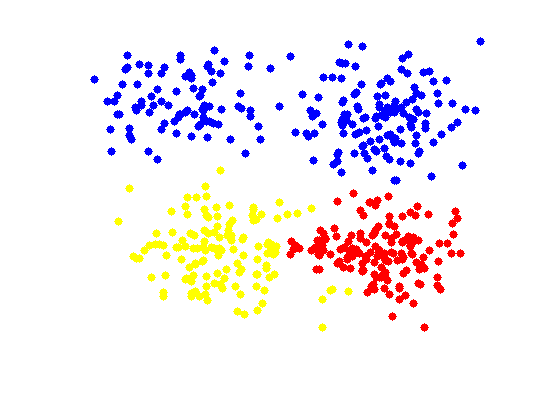}
               \end{tabular}
      \vspace{-8mm}
          \caption{\label{fig:multi_toy} Left: ground-truth, middle: clustering obtained by unconstrained
$1$-spectral clustering, right: clustering obtained by the constrained version.}
     \end{figure}
     
%     real world mulit-partitioning problems like USPS. The 
        
 \section{Experiments}
% \begin{table}\label{tab:UCI}
%\begin{center}
%\begin{tabular}{c|c|c|c|c|c|c|c|c}%{|l|c|c|}
%%\hline
% Dataset 	 & Hepatitis & Sonar & Heart & Ionosphere & WDBC & Diabetis & Spam & USPS \\
%\hline
% Size 	 & 80 & 208 & 270 & 351& 569 & 768 & 4207 & 9298\\
% \hline
% Features  &  19 & 60 & 13 & 34 & 30 & 8 & 57 & 256\\
%%\hline
%\end{tabular}
%\caption{\label{tab:UCI} Summary of datasets considered in the experiments.}
%\end{center}
%\end{table}

We compare our method  against the 
following four related constrained clustering approaches: Spectral Learning (SL) \cite{KamKleMan2003},
 Flexible Constrained Spectral Clustering (CSP) \cite{WanDav2010}, 
  Constrained Clustering via Spectral Regularization (CCSR) \cite{LiLiuTan2009} and Spectral
  Clustering with Linear Constraints (SCLC) \cite{XuLiSch2009}. 
  %For the CSP and CCSR we use the code provided by the authors on their webpages. 
 SL integrates the constraints by simply modifying the weight matrix such that the edges connecting must-links have maximum weight 
 and the edges of cannot-links have zero weight. 
 CSP starts from the spectral relaxation and restricts the space of feasible solutions %, using a non-convex quadratic constraint, 
 to those that satisfy a certain amount (specified by the user) of constraints. 
 This amounts to solving a full generalized eigenproblem %(where the matrix involved is not necessarily positive semi-definite) 
  and choosing among the eigenvectors corresponding to positive eigenvalues the one 
 that has minimum cost. 
% CSP, which handles constraints with varying degree of belief, considers the unconstrained spectral clustering problem and enforces the constraint information explicitly in a way similar to our soft constrained version. However, CSP  is applicable only for 2-class problems and has a free parameter which bounds the minimum amount of constraints satisfied. 
% This parameter is fixed in their code and we use the default value. 
 CCSR addresses the problem of incorporating the constraints in the multi-class problem directly by an SDP
which aims at adapting the spectral embedding to be consistent with the constraint information. 
% The dimension, $m$, of the embedding space is a parameter of this method. However, as we have little intuition about this parameter
% we use again the default value provided in the code.
 For CSP and CCSR we use the code provided by the authors on their webpages.
 
  In SCLC one solves the spectral relaxation of the normalized cut problem subject to linear constraints \cite{EriOlsKah2007, XuLiSch2009}.
  %we use the fast projected power method proposed in \cite{XuLiSch2009} in our experiments.
%  Similar problem is solved by \cite{EriOlsKah2007} by 
  %, XuLiSch2009}
 Cannot-links and must-links are encoded via linear constraints as follows \cite{EriOlsKah2007}: if the vertices $p$ and $q$ cannot-link
  (resp. must-link) then add a constraint $f_p = -f_q$ (resp. $f_p = f_q$).
   Although must-links are correctly formulated, one can argue that
   the encoding of cannot-links has modeling drawbacks.
    First observe that any solution that assigns zero to the constrained vertices $p$ and $q$ still 
    satisfies the corresponding cannot-link constraint although it is not feasible to the constrained cut problem. 
    Moreover, one can observe from the derivation of spectral relaxation \cite{Lux2007}, that 
    vertices belonging to different components need to have only different signs but not the same value.  
    Encoding cannot-links this way introduces bias towards partitions of equal volume, which can be observed in the experiments. 

 Our evaluation is based on three criteria: clustering error, normalized cut and fraction of constraints violated. For the clustering error we take the known labels and classify each cluster 
using majority vote. In this way each point is assigned a label and the clustering error is the error of this labeling. We use this measure as it is the expected error one would obtain when using 
simple semi-supervised learning, where one labels each cluster using majority vote.  

%The setting in the experiments are as follows. 
%We generate $k$ pairs of constraints starting from $k=5$ until the constraints fully specify the labeling. 
%where $k \in \{5,10,20,40,80,160,320,640,1280,2560\}$ for all datasets except spam and USPS. 
% For spam we use $k \in \{1000, 1500, 2000, 2500, 3000, 3500, 4000, 4500, 5000 \}$ and for USPS $k \in \{50, 75, 100, 125, 150, 175, 200\}$ 
% since these are the regions where the constrained solution deviates from the unconstrained solution. 
% and USPS we sampled the region where the constrained solution starts to differ from the and 

The summary of the datasets considered is given in 
Table \ref{tab:UCI}. 
 The data with missing values are removed 
 and the $k$-NN similarity graph is constructed from the remaining data as in \cite{BueHei2009}. 
% Moreover, redundant data points are removed from the spam dataset.
 % We report results for the normalized cut problem with cannot-links. 
\begin{table}
\begin{center}
\begin{tabular}{|l|c|c|c|}
\hline
 Dataset       &  Size & Features & Classes\\
\hline
% Hepatitis   &80   &19\\
 Sonar         & 208   & 60  & 2\\
%  Breast Cancer & 263   & 9   & 2\\  
 % Heart         & 270   & 13  & 2\\
%  Ionosphere          &351    & 34  \\
% WDBC          & 569   & 30  & 2\\ 
 % Diabetis      & 768   & 8   & 2\\
  Spam          & 4207 & 57   & 2\\
 USPS          & 9298  & 256   & 10\\
 Shuttle 	   & 58000 & 9 & 7\\
% MNIST	    &70000 & 784 &10\\
  MNIST (Ext)%\footnote{The extended MNIST dataset is generated by translating each original input image of MNIST by one pixel, 
  %i.e., 8 directions.}  
   &630000 & 784 &10\\
\hline
\end{tabular}
\caption{\label{tab:UCI} UCI datasets. The extended MNIST dataset is generated by translating each original input image of MNIST by one pixel, 
  i.e., 8 directions.}
\end{center}
\end{table}
% For all datasets, graphs are constructed as in \cite{HeiBue2010}. 
In order to illustrate the performance in case of highly unbalanced problems, 
 we create a binary problem (digit 0 versus rest) from USPS.
The constraint pairs are generated in the following manner.
We randomly sample pairs of points and for each pair, we introduce a cannot or must-link constraint based on the labels of the sampled pair. 
The results, averaged over $10$ trials are shown in Table \ref{tab:Plots} for 
 2-class problems and in Table \ref{tab:Plots_multi} for multi-class problems\footnote{
  CSP could not scale to these large datasets, as the method solves the full (generalized) eigenvalue problem where 
the matrices involved are not sparse. 
}. 
 %(due to space constraints some results have to be moved to the Appendix; see Table \ref{app:tab:Plots}).
 In the plots our method is denoted as COSC and we enforce always all constraints (see discussion in Section \ref{sec:soft_ver}).
 Since our formulation is a non-convex problem, we use the best result (based on the achieved cut value) of 10 runs with 
 random initializations.
 Except our method, no other method can guarantee to satisfy all constraints, even though SCLC does so in
all cases.
% Moreover, as we increase the number of constraints the NCutLinear problem actually becomes infeasible 
% as the
 Our method produces always much better cuts than the ones found by SCLC
  which shows that our method is better suited for solving the constrained normalized cut problem. 
 In terms of the clustering error, our method is consistently better than other methods.
  In case of unbalanced datasets (Spam, USPS 0 vs rest) our
  method significantly outperforms SCLC in terms of cuts and clustering error. 
 Moreover, because of  hard encoding of constraints, CSLC cannot solve multi-partitioning problems. 

\begin{table*}
\begin{tabular}{ccc}
%	 \includegraphics[width=0.32\textwidth]{Results/ErrorsBC.png}
%  &\includegraphics[width=0.32\textwidth]{Results/CutsBC.png}
%	&\includegraphics[width=0.32\textwidth]{Results/ConsBC.png}\\
\includegraphics[width=0.31\textwidth]{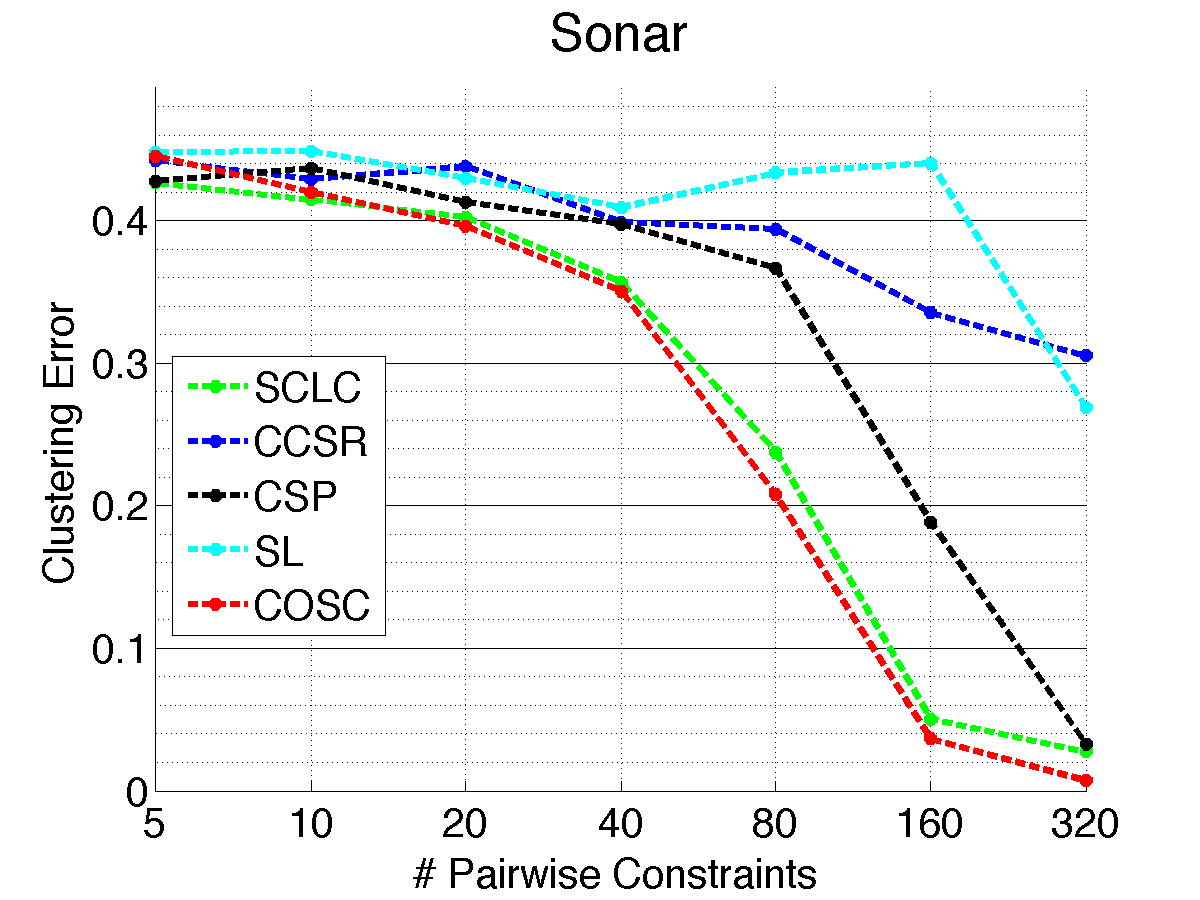}
&  \includegraphics[width=0.31\textwidth]{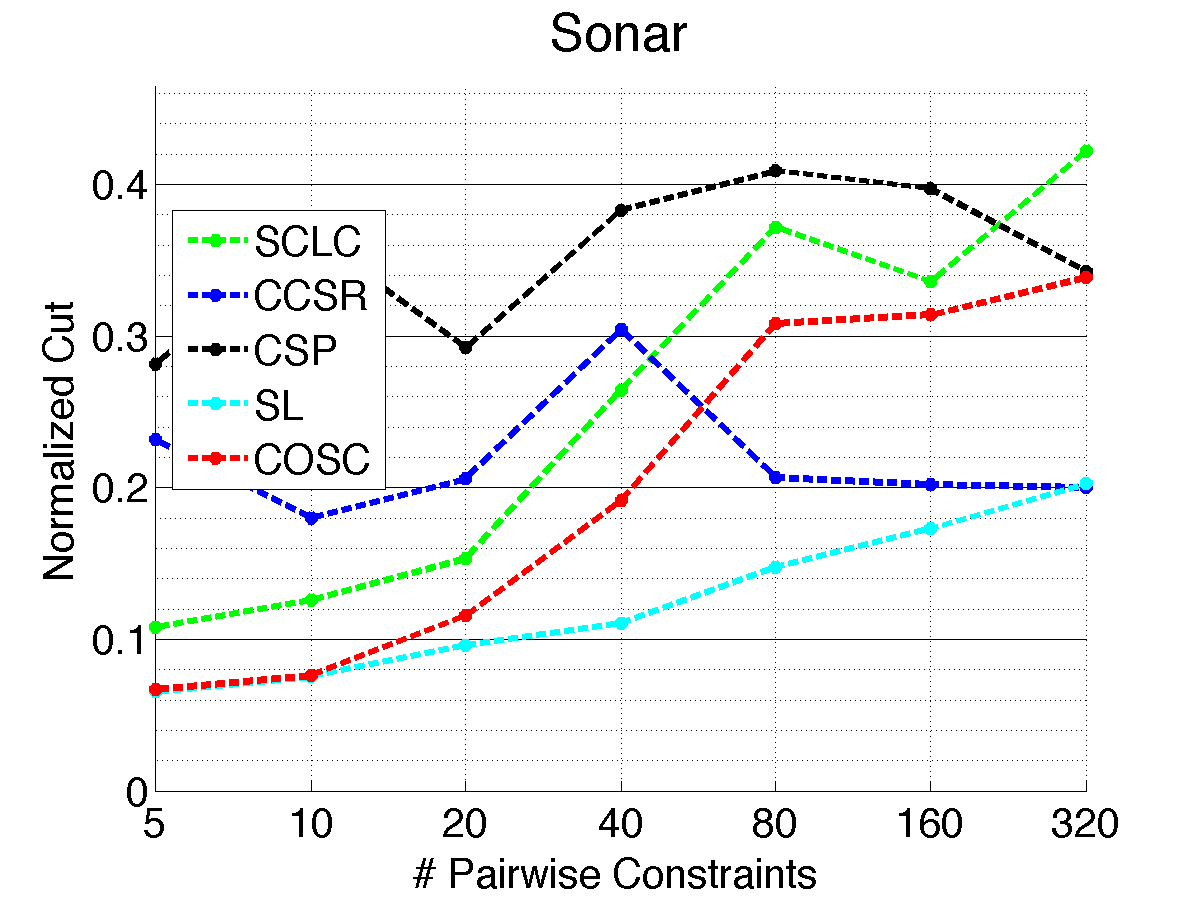}
  	&\includegraphics[width=0.31\textwidth]{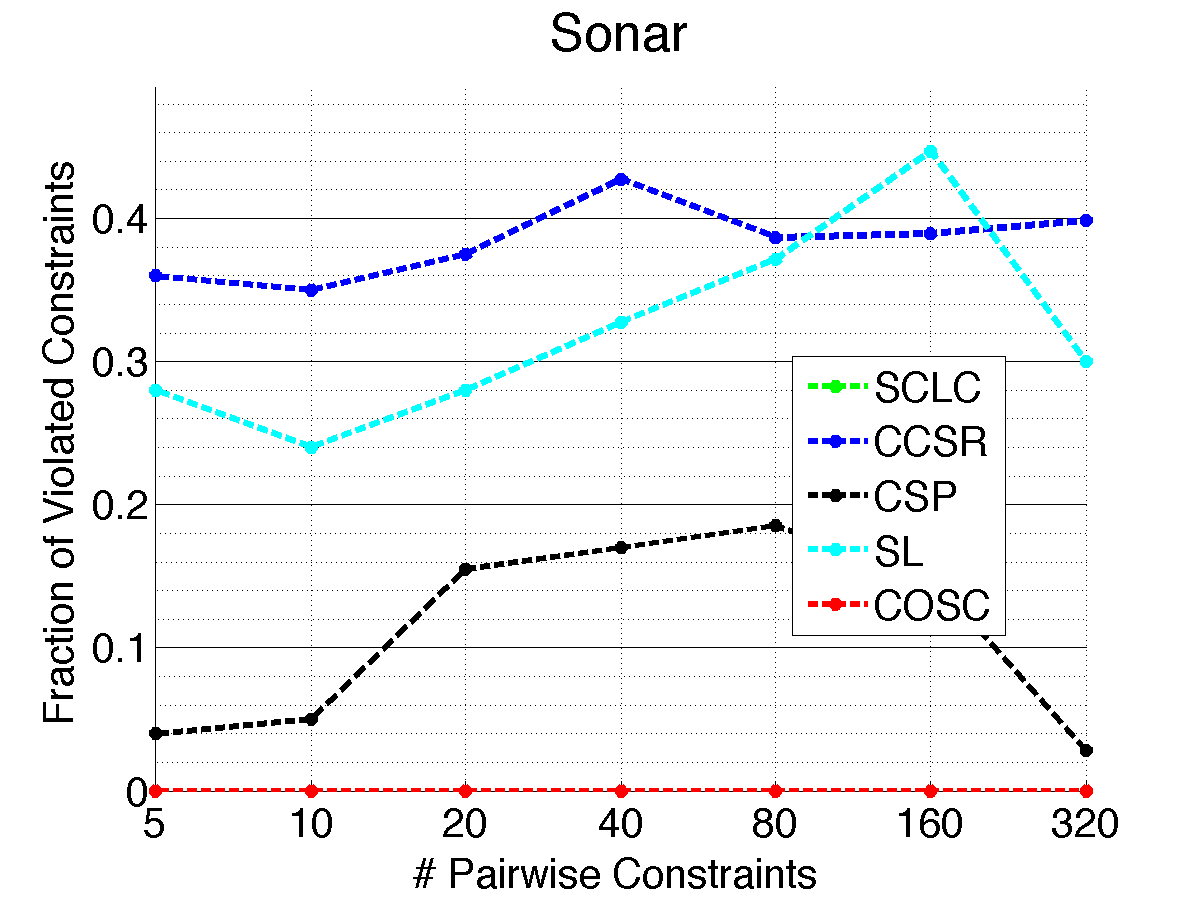}\\
%\end{tabular}
%\begin{tabular}{ccc}
%	\includegraphics[width=0.31\textwidth]{Results/plot_wdbc_errors.png}
%	&\includegraphics[width=0.31\textwidth]{Results/plot_wdbc_cuts.png}
%	&\includegraphics[width=0.31\textwidth]{Results/plot_wdbc_viols.png}\\
%		&\includegraphics[width=0.31\textwidth]{Results/plot_ion_errors.png}
%  &\includegraphics[width=0.31\textwidth]{Results/plot_ion_cuts.png}
%	&\includegraphics[width=0.31\textwidth]{Results/plot_ion_viols.png}\\
%  \includegraphics[width=0.31\textwidth]{Results/plot_diabetis_errors.png}
% & \includegraphics[width=0.31\textwidth]{Results/plot_diabetis_cuts.png}
%	&\includegraphics[width=0.31\textwidth]{Results/plot_diabetis_viols.png}\\
	%	\includegraphics[width=0.31\textwidth]{Results/ErrorsHEART_NIPS.png}
%  &\includegraphics[width=0.31\textwidth]{Results/CutsHEART_NIPS.png}
%	&\includegraphics[width=0.31\textwidth]{Results/ConsHEART_NIPS.png}\\
  	\includegraphics[width=0.31\textwidth]{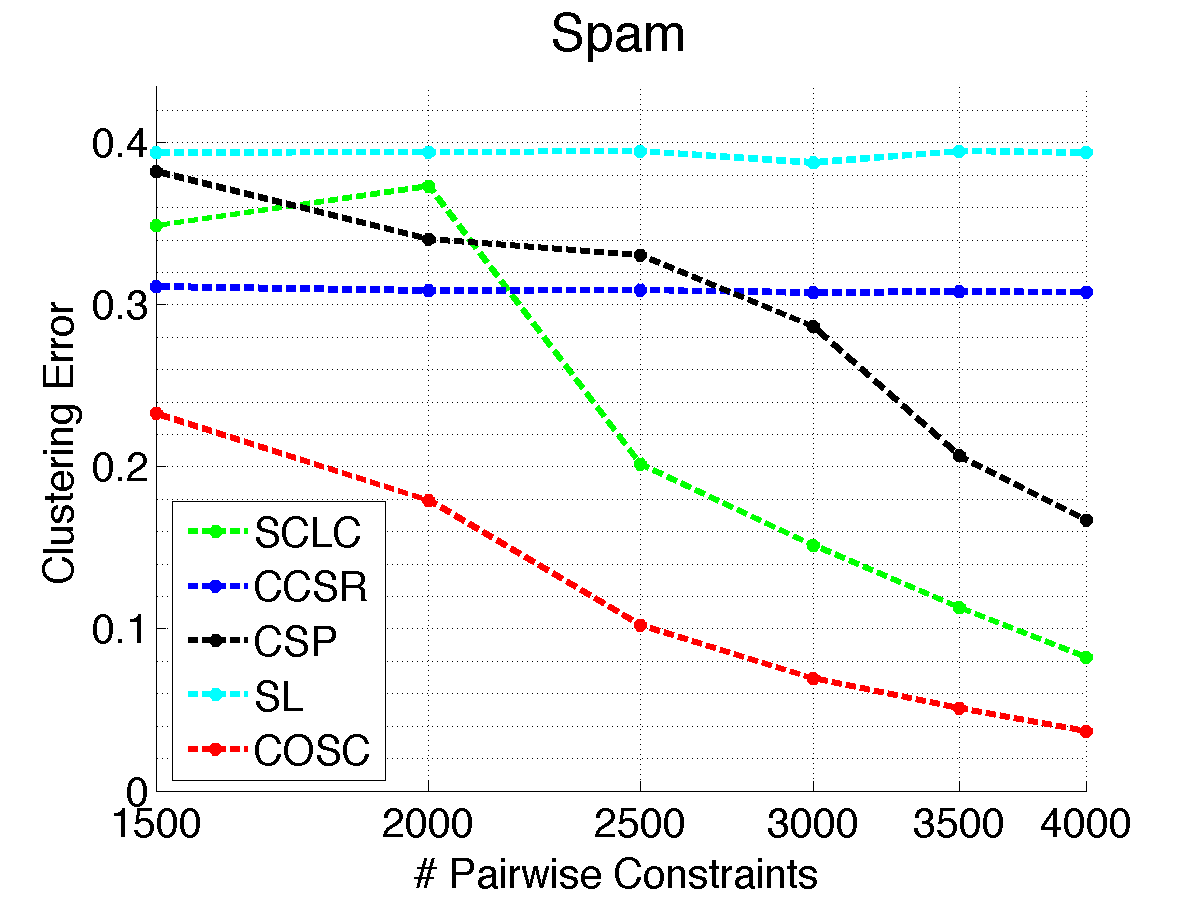}
	 & \includegraphics[width=0.31\textwidth]{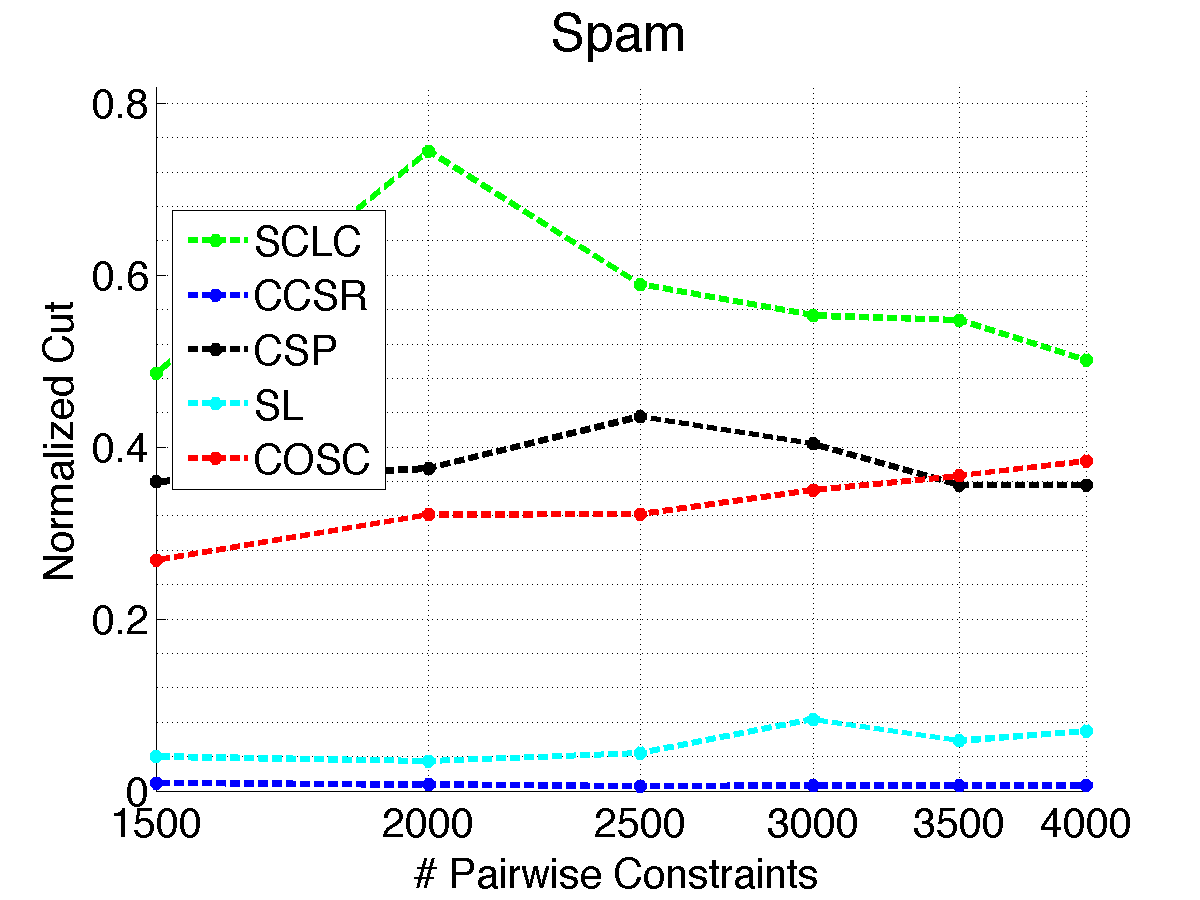}
	&\includegraphics[width=0.31\textwidth]{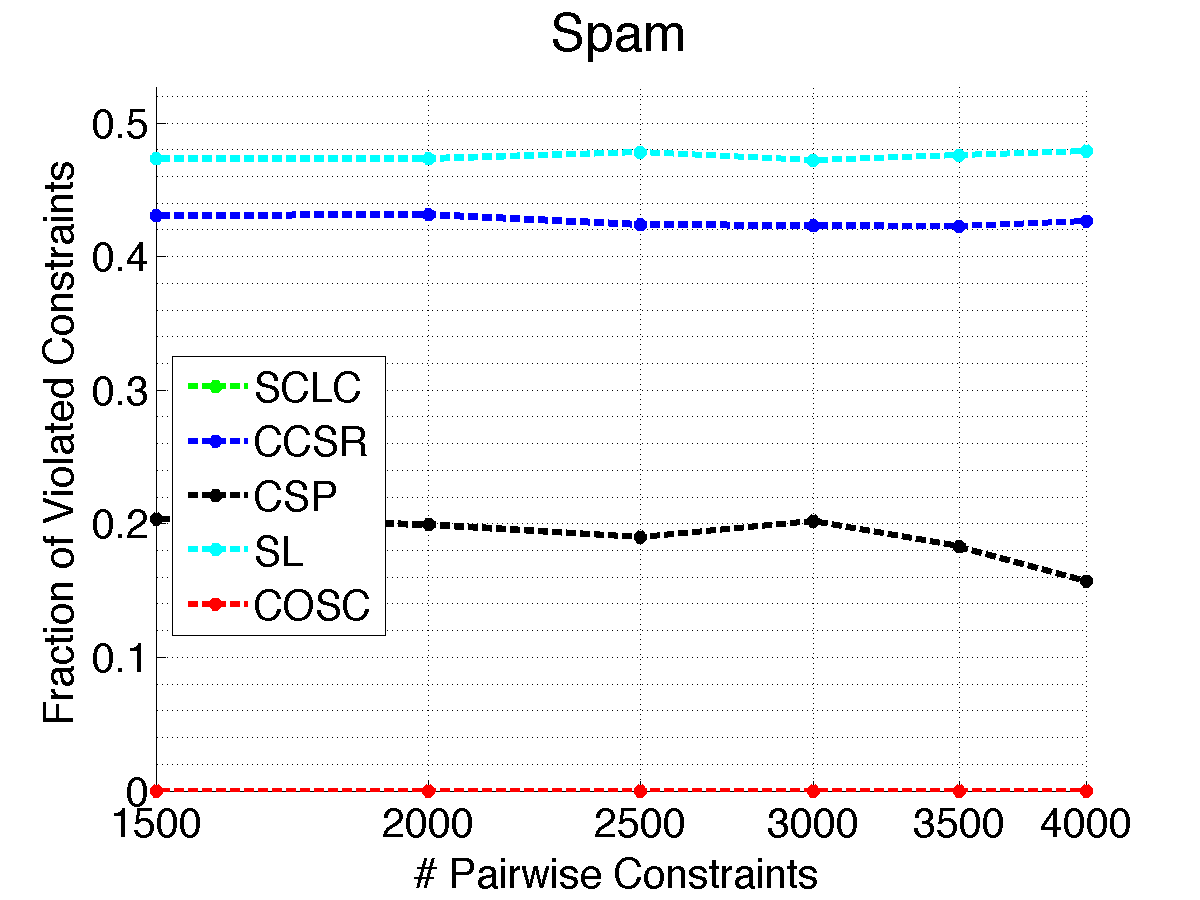}\\
  \includegraphics[width=0.31\textwidth]{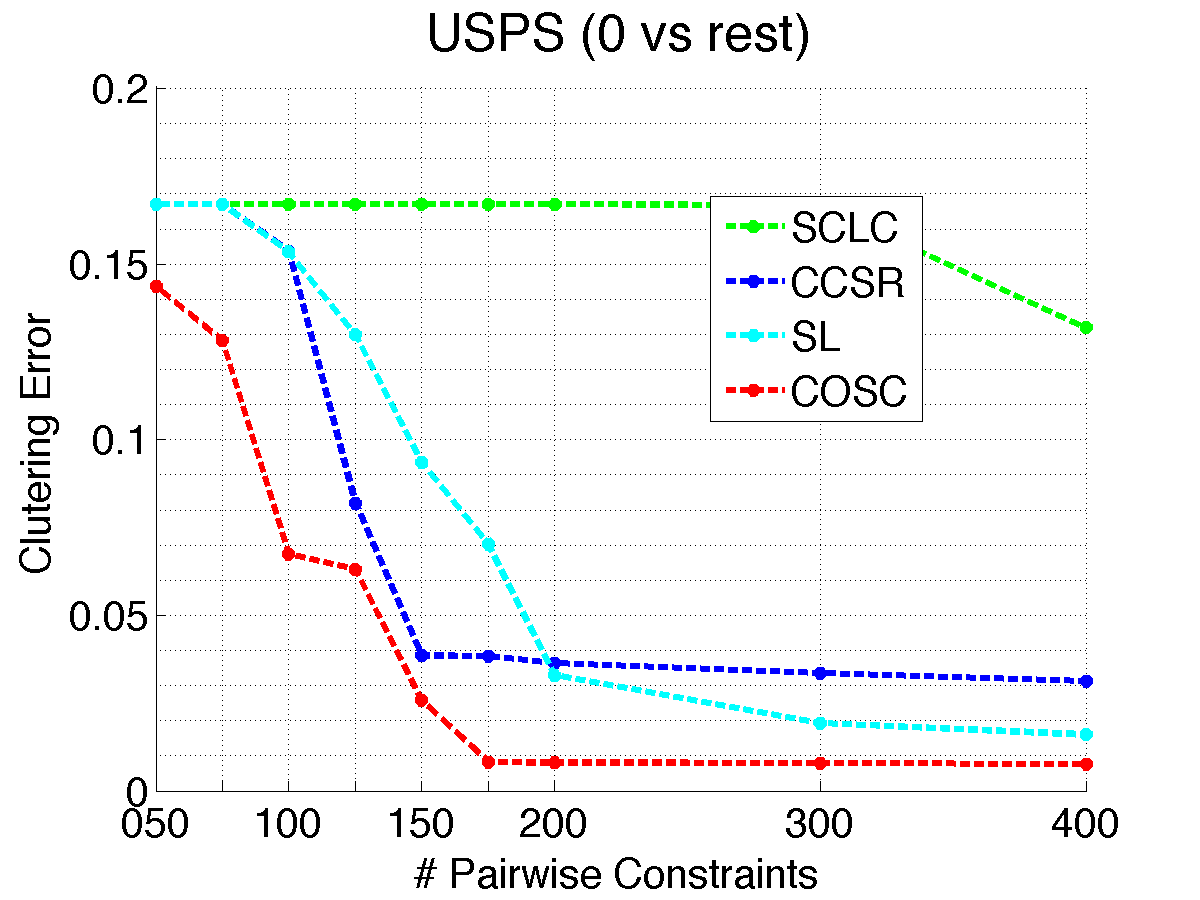}
  &  \includegraphics[width=0.31\textwidth]{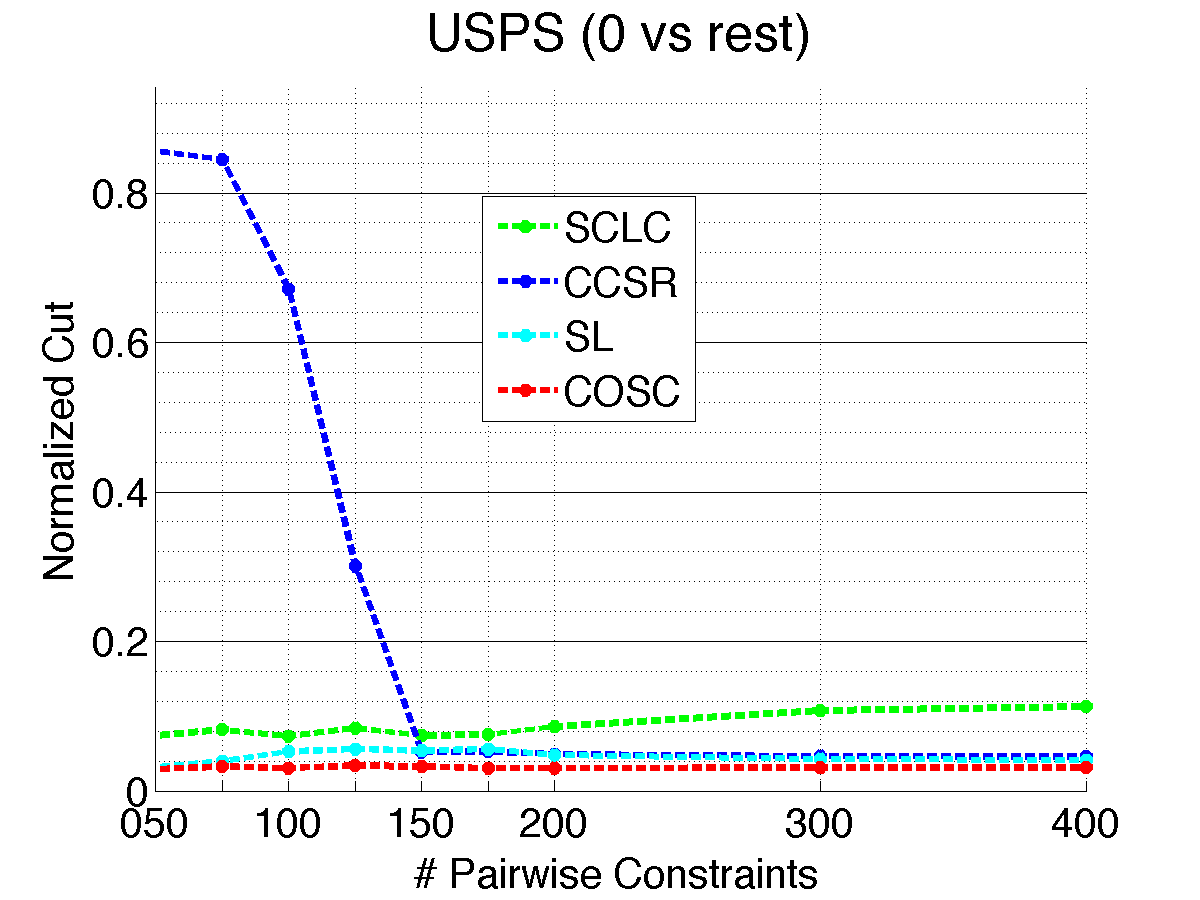}
	&\includegraphics[width=0.31\textwidth]{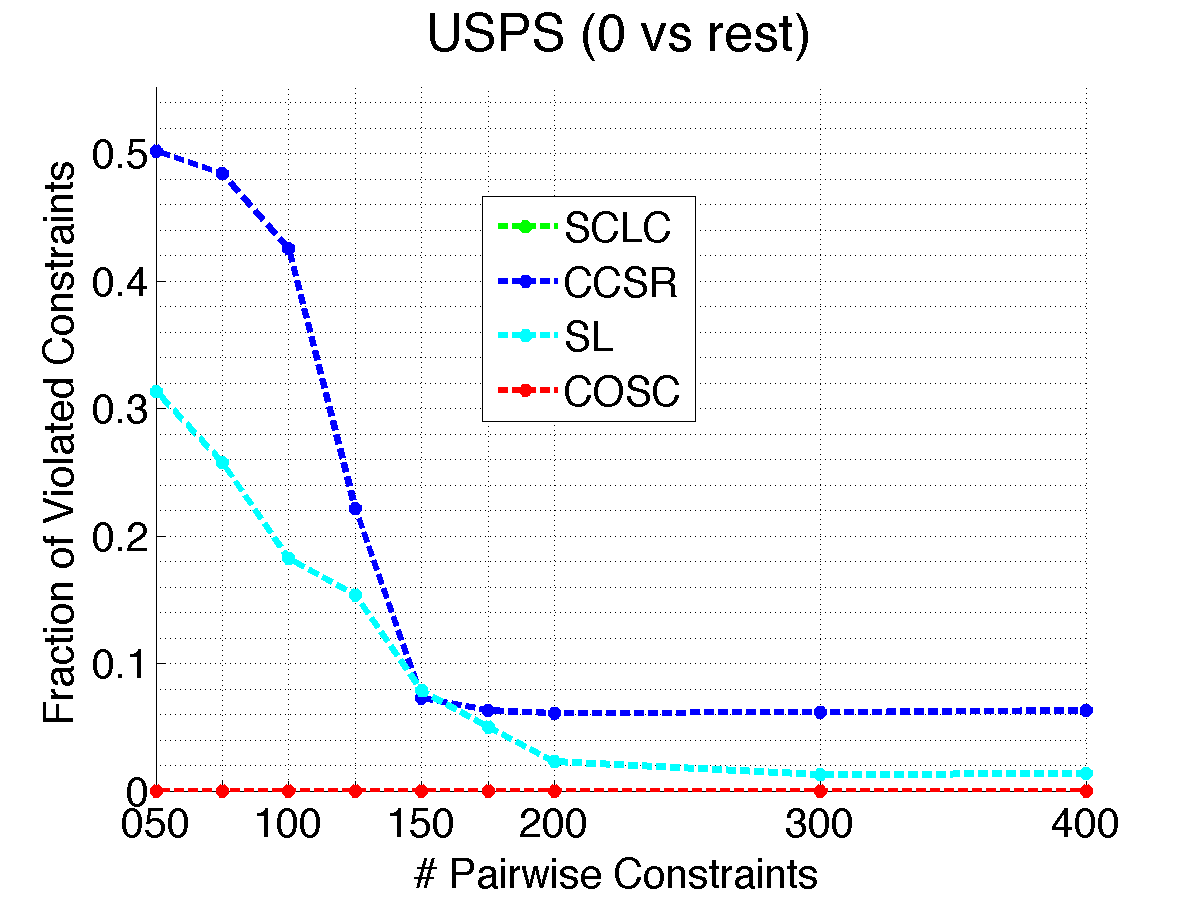}\\
%	\includegraphics[width=0.31\textwidth]{Results/plot_mnist_0_vs_rest_cuts.png}
%	&\includegraphics[width=0.31\textwidth]{Results/plot_mnist_0_vs_rest_errors.png} 
%	&\includegraphics[width=0.31\textwidth]{Results/plot_mnist_0_vs_rest_viols.png}\\
	\end{tabular}
\caption{\label{tab:Plots}Results for \textbf{binary partitioning}: Left: clustering error versus number of constraints, Middle: normalized cut versus number of constraints, Right: fraction of violated constraints versus number of constraints. }
%Our Method CBCut outperforms for almost all settings the competing methods CCSR, CSP and SL in terms of clustering 
%    error and fraction of violated constraints.}
\end{table*}
 \begin{table*}
\begin{tabular}{ccc}
	\includegraphics[width=0.31\textwidth]{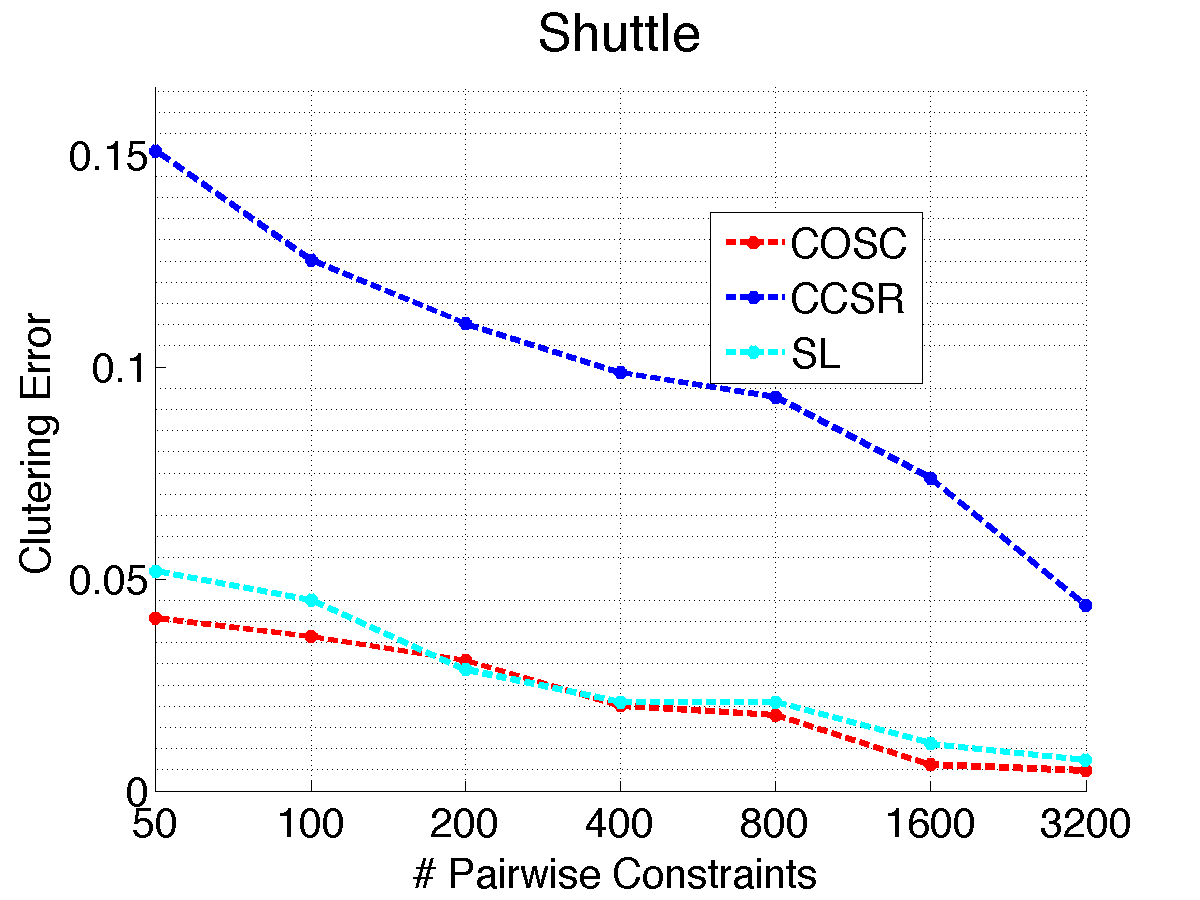}	
	  &\includegraphics[width=0.31\textwidth]{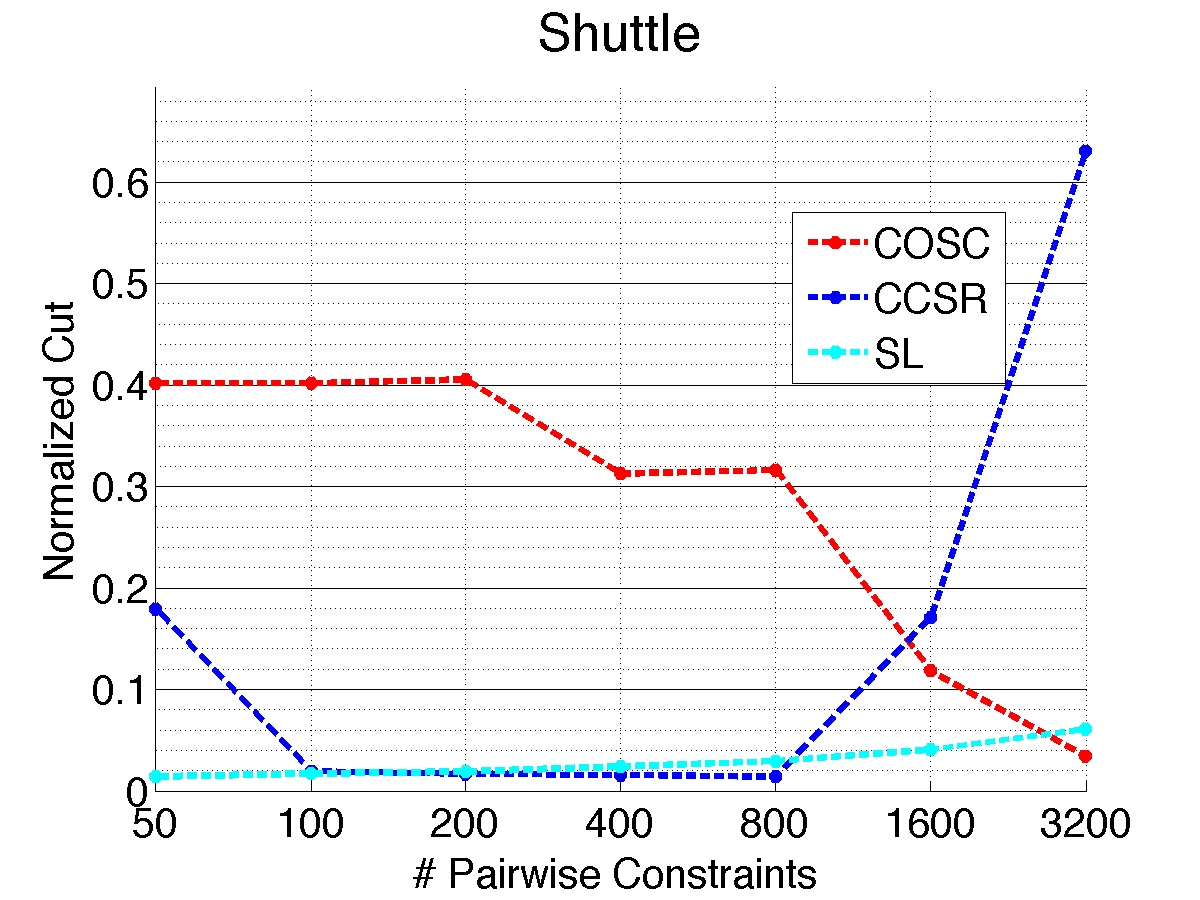}
	&\includegraphics[width=0.31\textwidth]{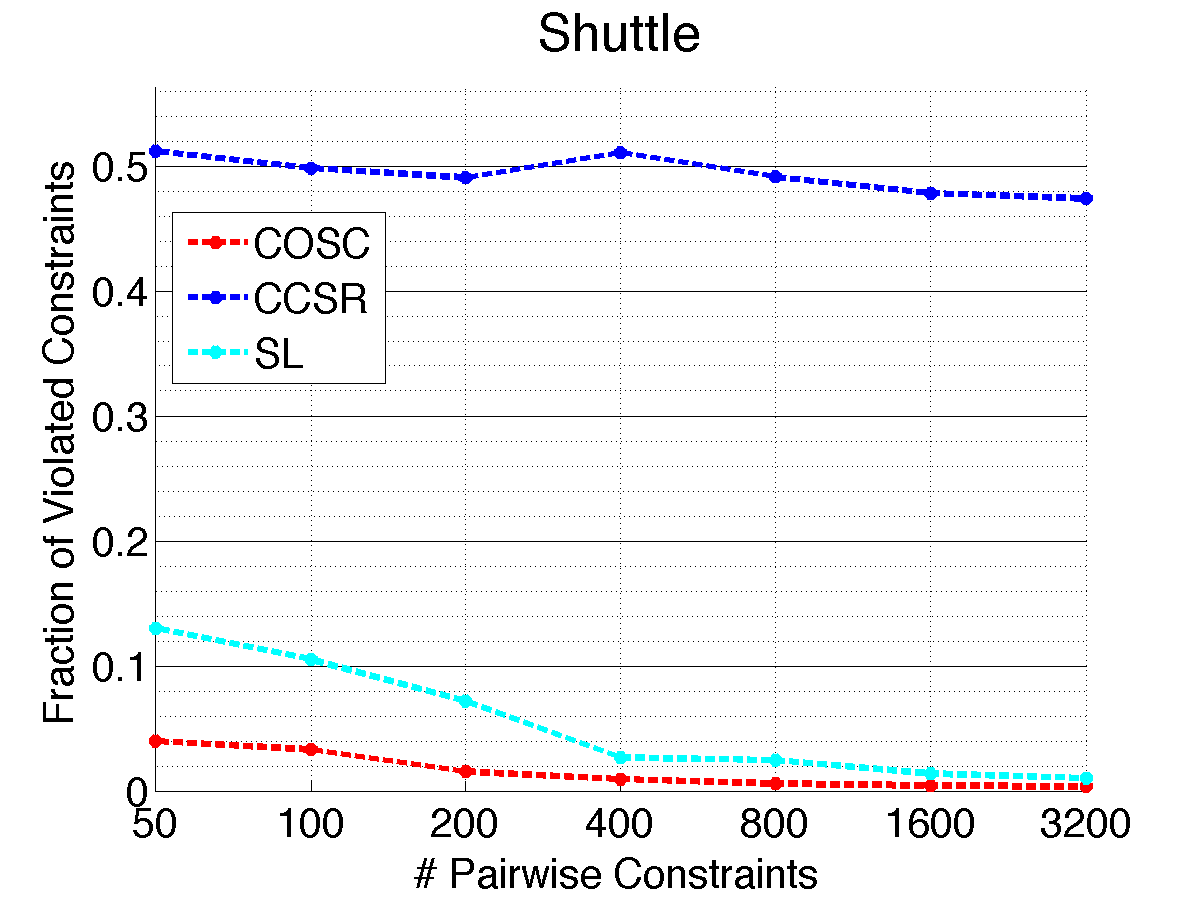}\\
\includegraphics[width=0.31\textwidth]{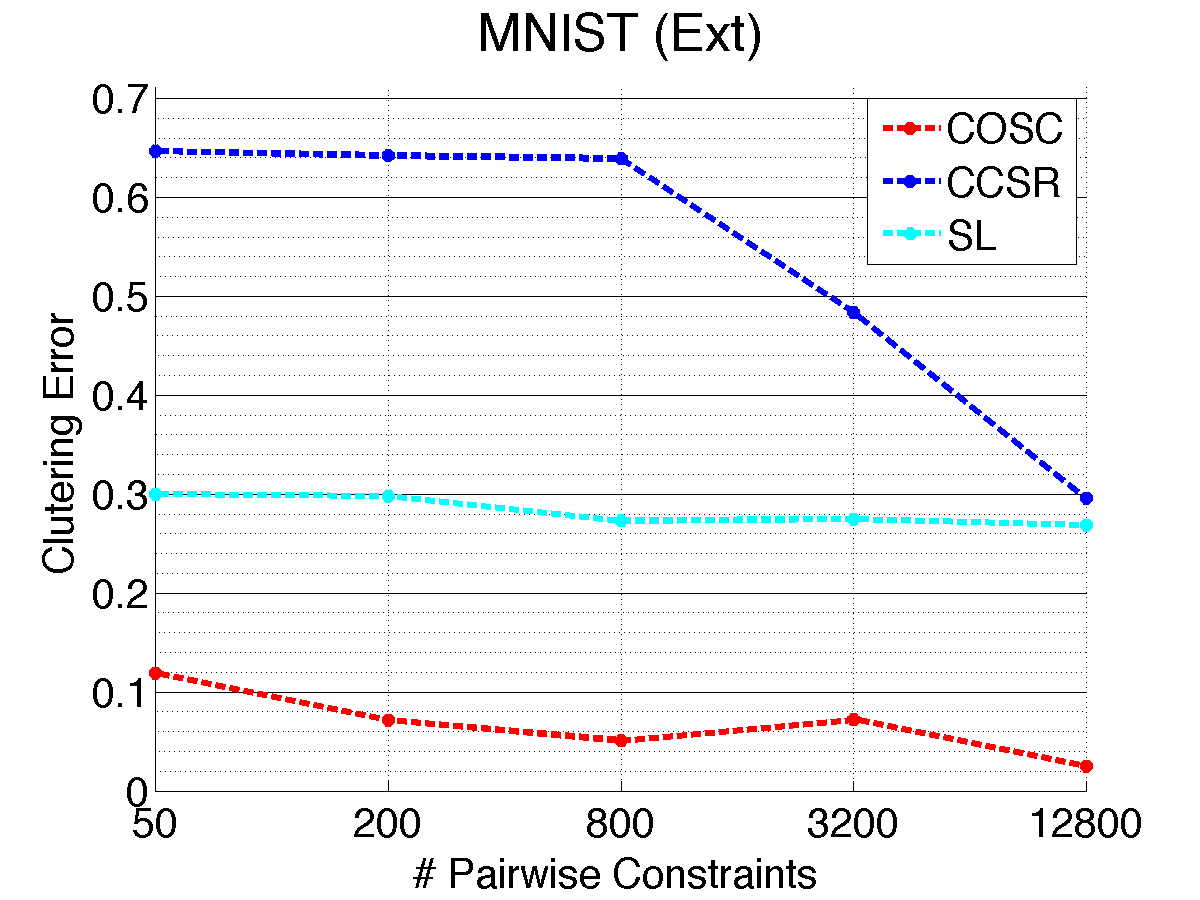}
&  \includegraphics[width=0.31\textwidth]{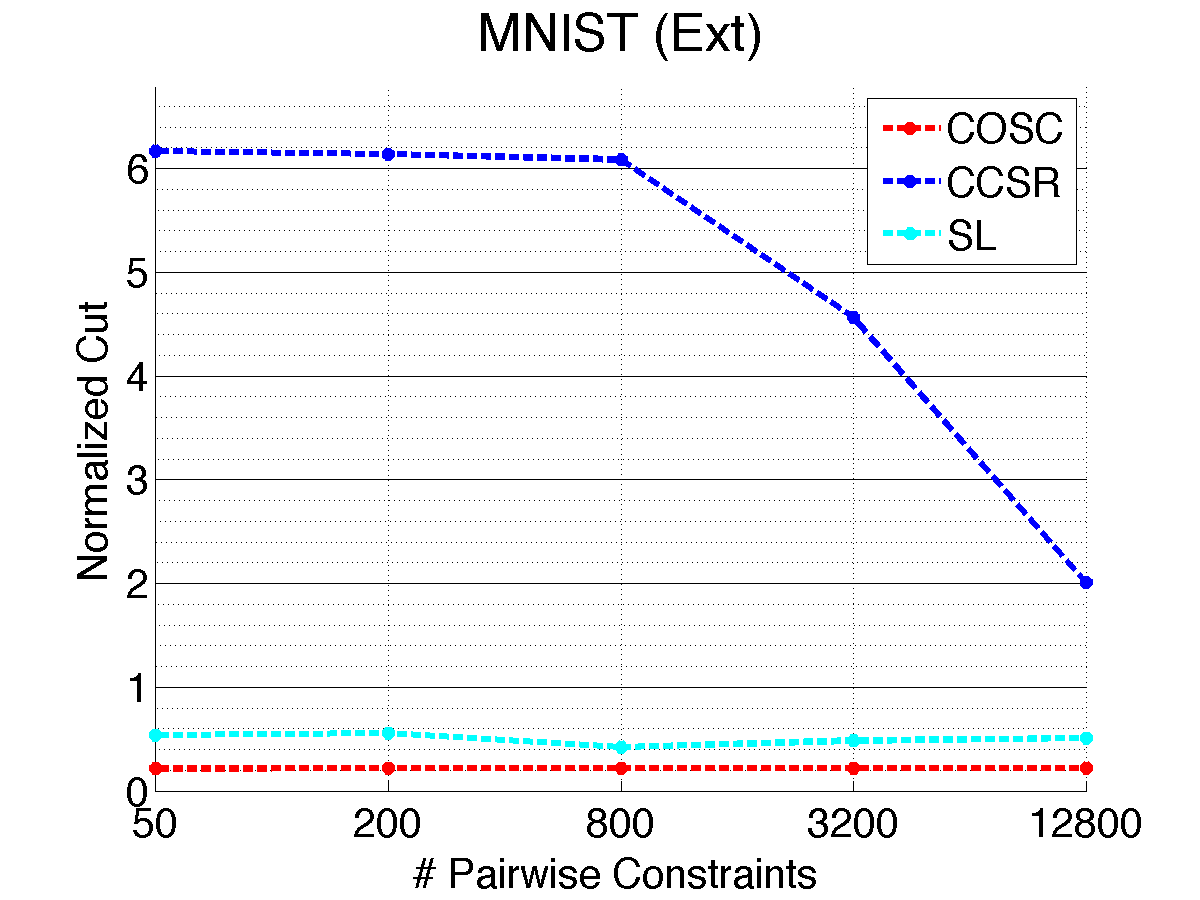}
	&\includegraphics[width=0.31\textwidth]{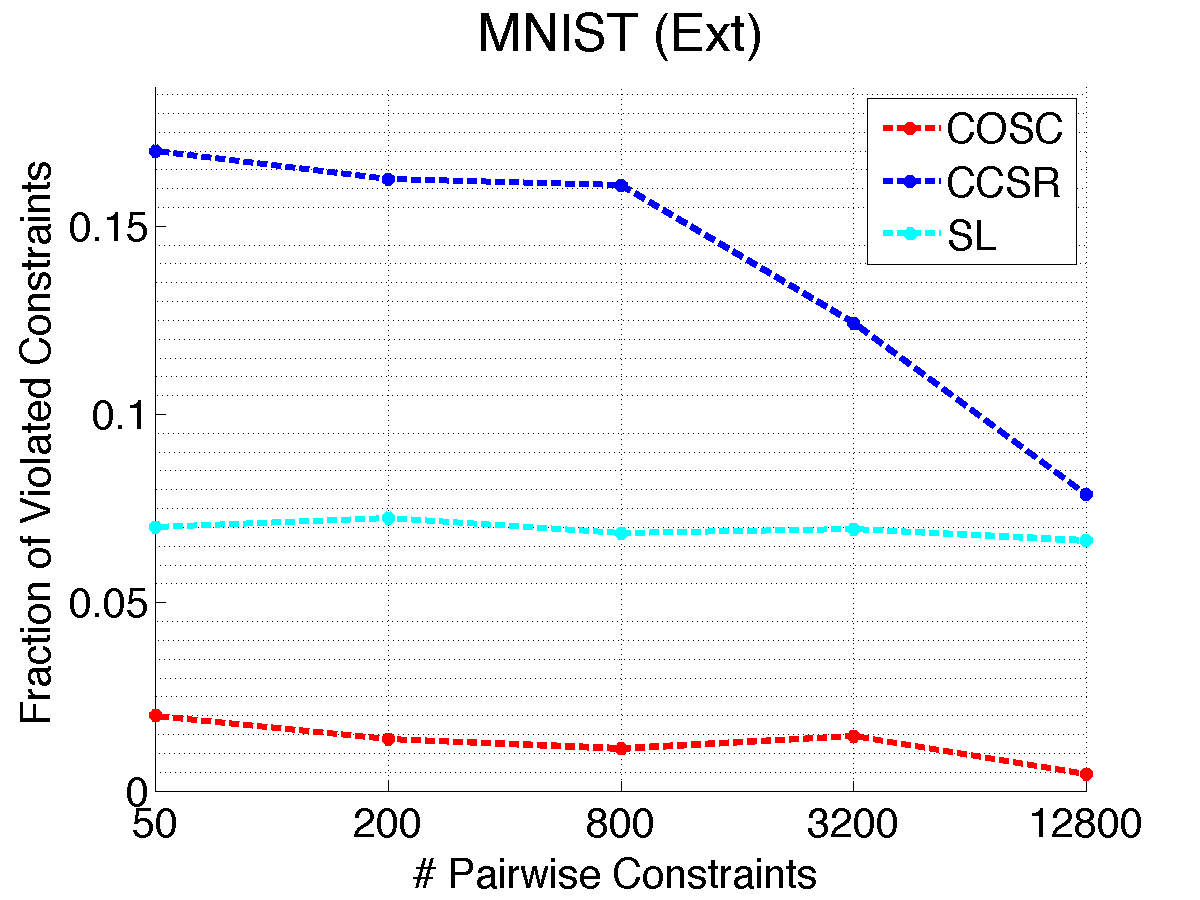}\\
	%	\includegraphics[width=0.31\textwidth]{Results/plot_mnist_errors.png}
%  &\includegraphics[width=0.31\textwidth]{Results/plot_mnist_cuts.png}
%	&\includegraphics[width=0.31\textwidth]{Results/plot_mnist_viols.png}\\
		\end{tabular}
\caption{\label{tab:Plots_multi}Results for \textbf{multi-partitioning} - Left: clustering error versus number of constraints, Middle: normalized cut versus number of constraints, Right: fraction of violated constraints versus number of constraints.}
\end{table*}
% The results for multi-partitioning are reported in table \ref{tab:Plots_multi}. Note that because of
%  hard encoding of constraints, CSLC cannot solve multi-partitioning problems. 
%  Since we use the soft version of our formulation for multi-partitioning, 
%  all constraints are not satisfied opposite to the binary case.
%  Again our method outperforms all other methods in terms of achieved cut, clustering error and 
%  number of constraints satisfied. 

\subsection{Additional experimental results}
 Additional experimental results are given in Tables \ref{app:tab:Plots} and   \ref{app:tab:Plots_multi} 
  for the datasets given in Table \ref{tab:UCI_additional}
 \begin{table}\label{tab:UCI_additional}
\begin{center}
\begin{tabular}{|l|c|c|c|}
\hline
 Dataset       &  Size & Features & Classes\\
\hline
% Hepatitis   &80   &19\\
 %Sonar         & 208   & 60  & 2\\
 Breast Cancer & 263   & 9  & 2 \\  
 Heart         & 270   & 13  & 2\\
%  Ionosphere          &351    & 34  \\
% WDBC          & 569   & 30  & 2\\ 
 Diabetis      & 768   & 8   & 2\\
%  Spam          & 4207 & 57   & 2\\
 USPS          & 9298  & 256   & 10\\
 MNIST	    &70000 & 784 &10\\
\hline
\end{tabular}
\caption{Additional UCI datasets}
\end{center}
\end{table}

\begin{table*}
\begin{tabular}{ccc}
	 \includegraphics[width=0.32\textwidth]{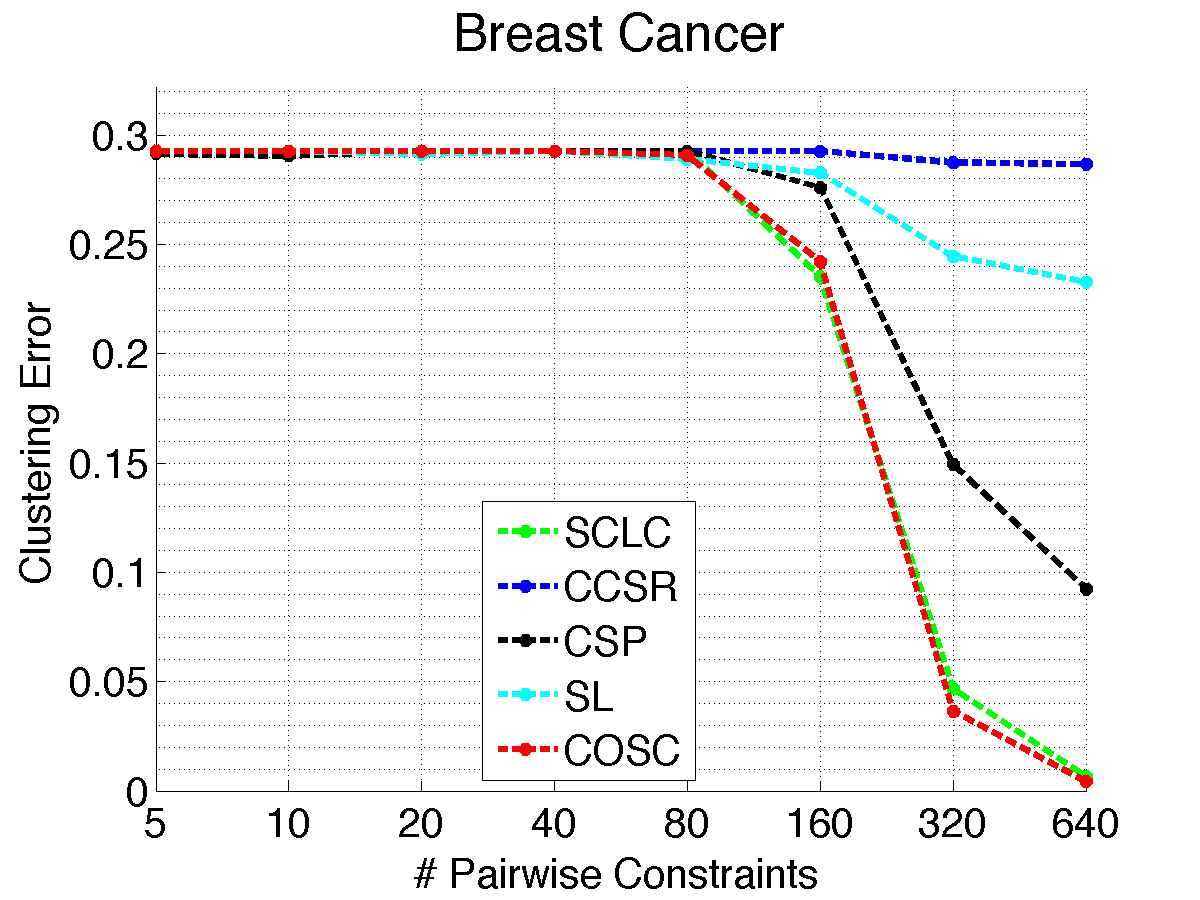}
  &\includegraphics[width=0.32\textwidth]{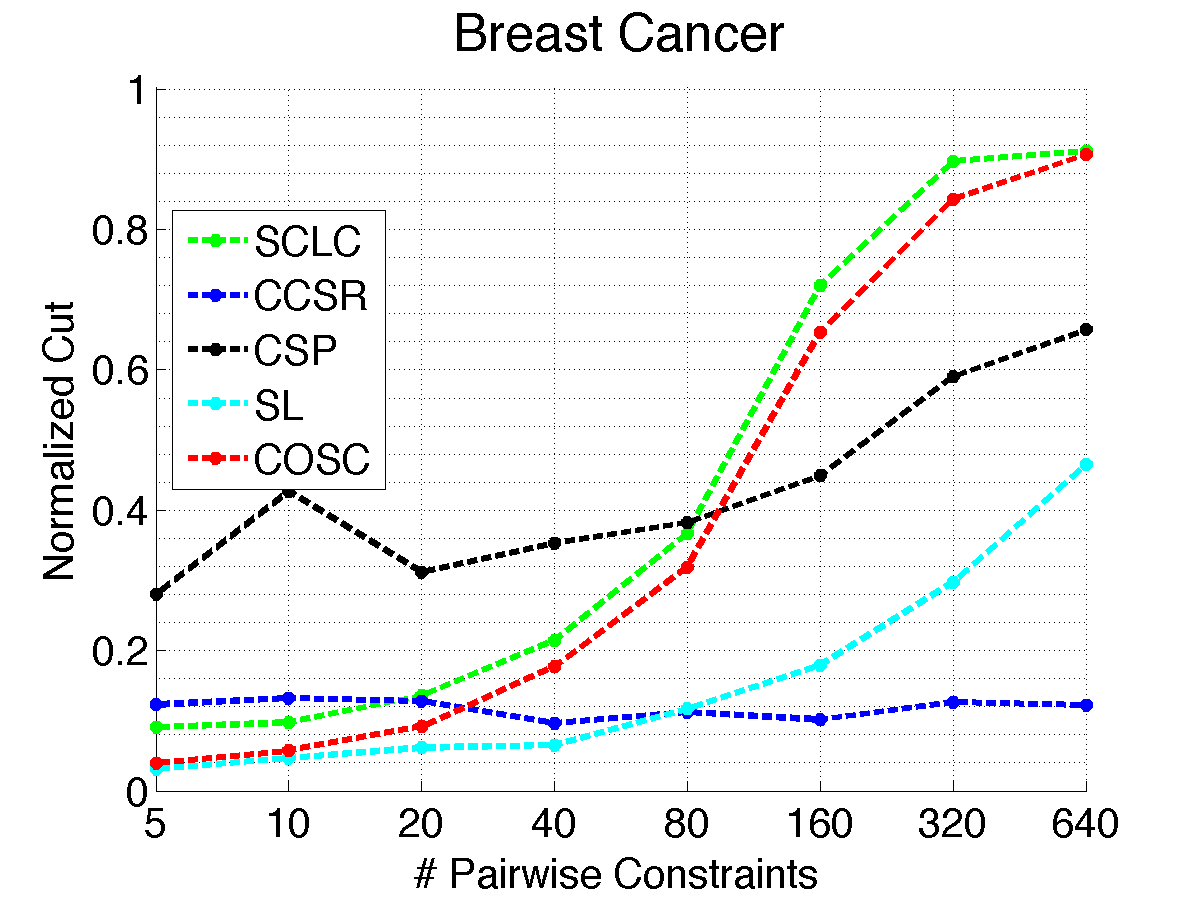}
	&\includegraphics[width=0.32\textwidth]{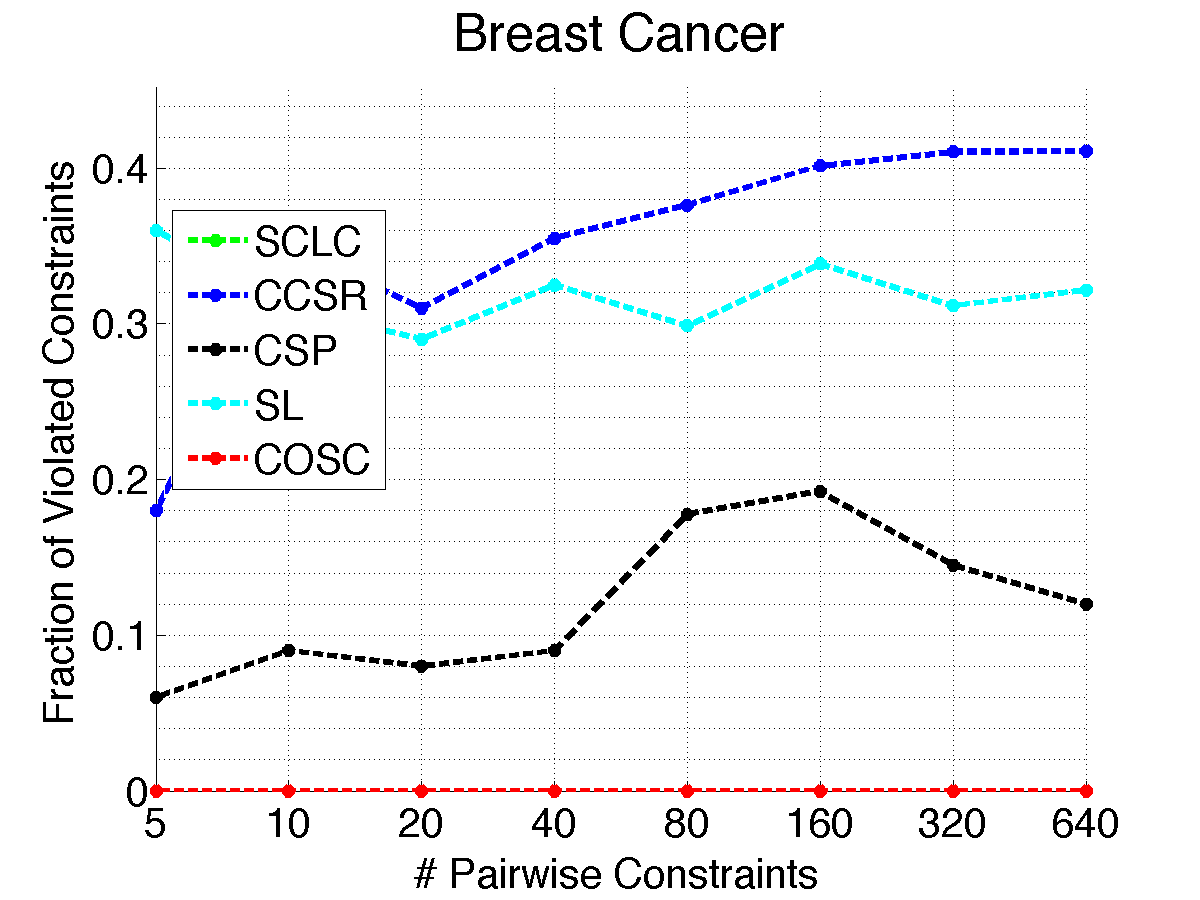}\\
	\includegraphics[width=0.31\textwidth]{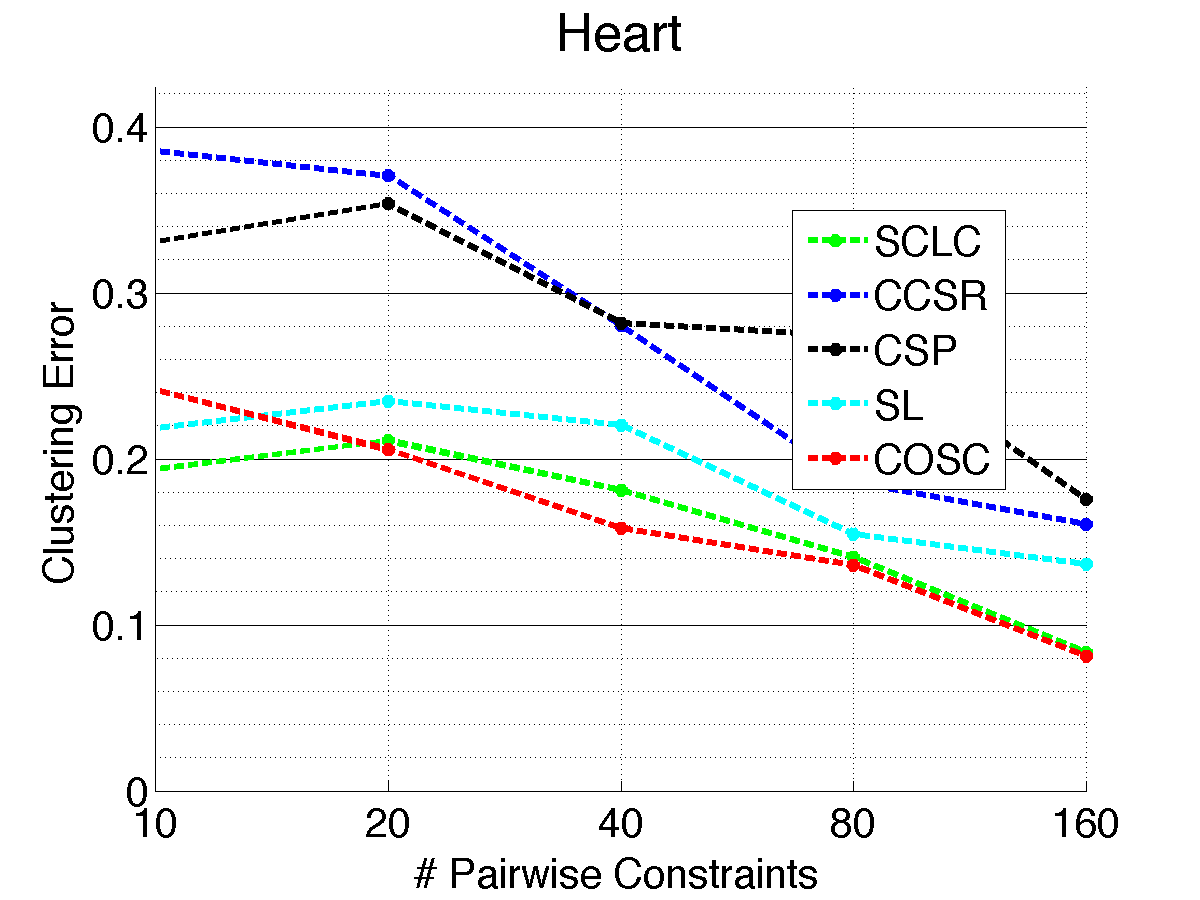}
&  \includegraphics[width=0.31\textwidth]{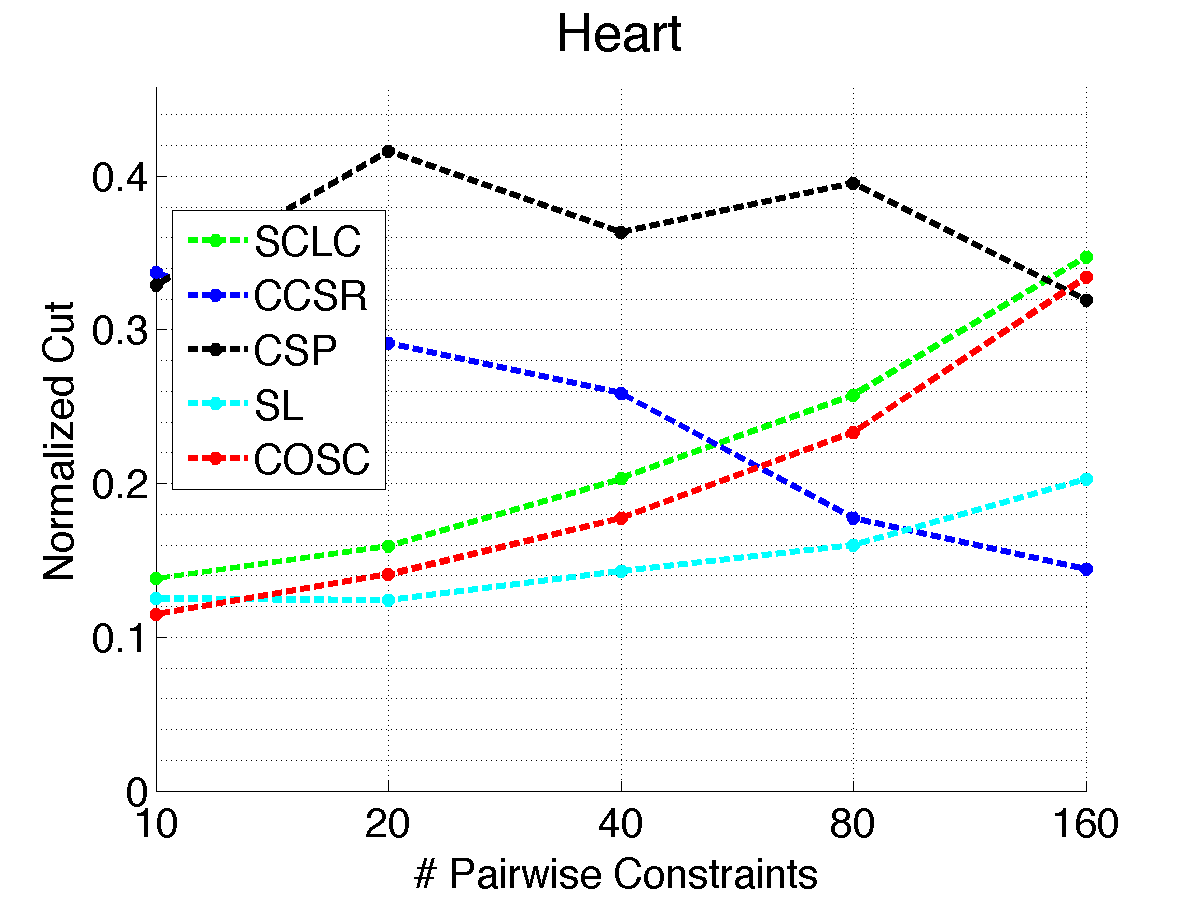}
  	&\includegraphics[width=0.31\textwidth]{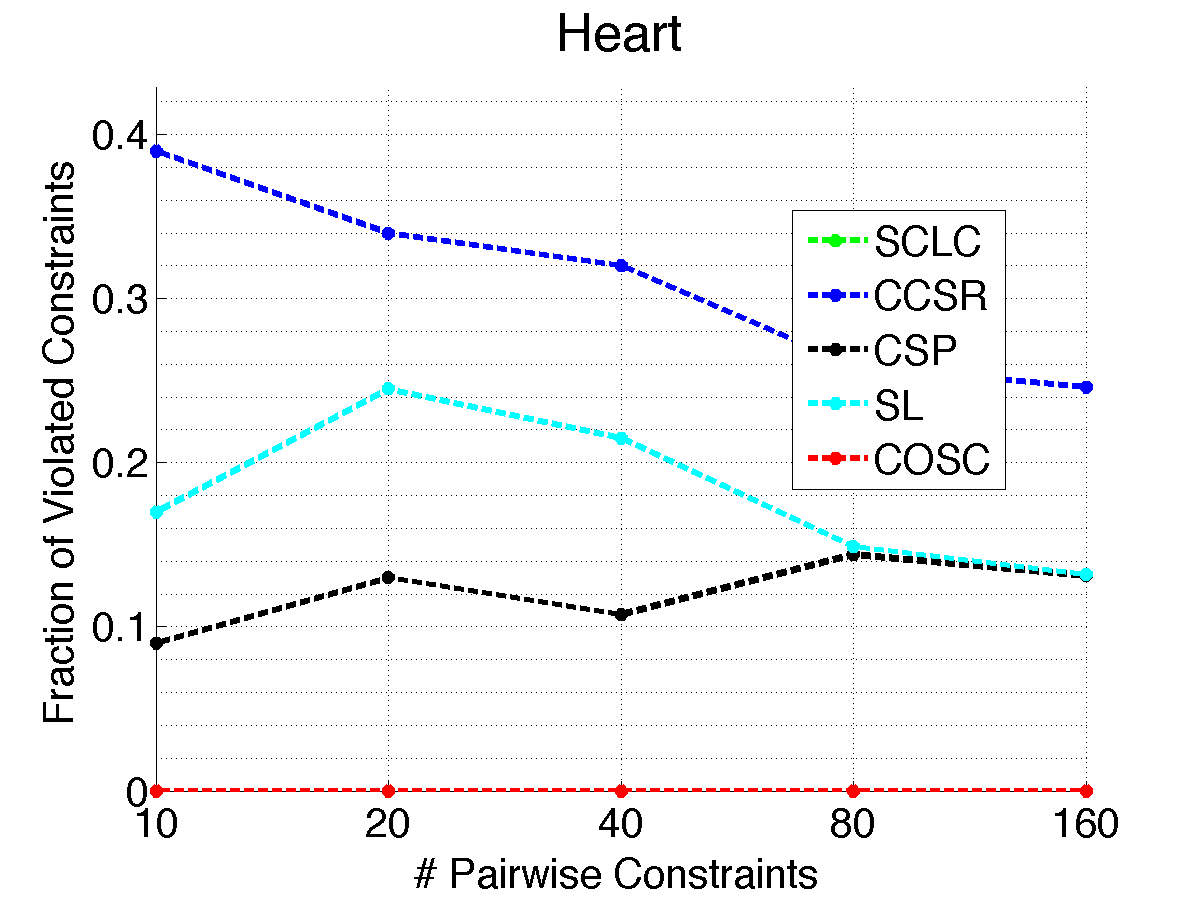}\\
\includegraphics[width=0.31\textwidth]{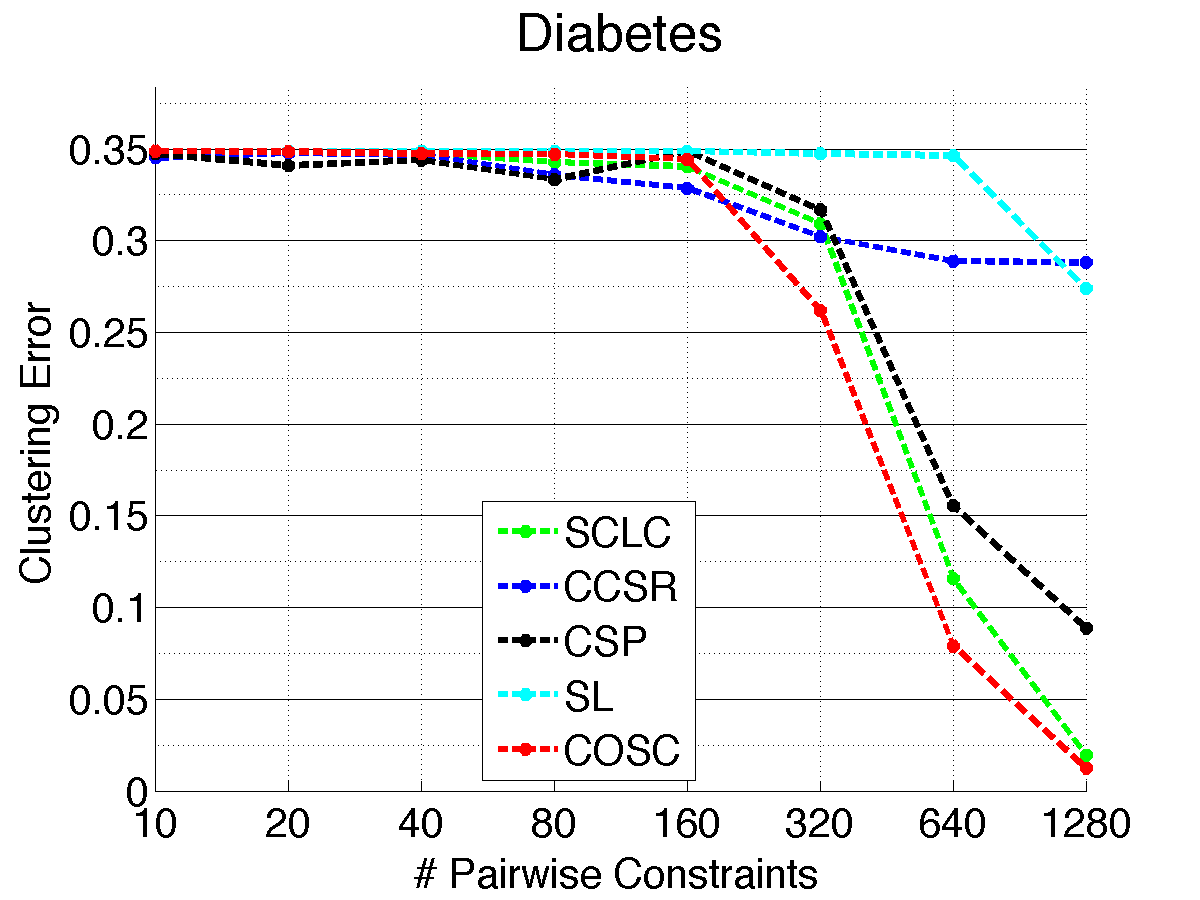}
&  \includegraphics[width=0.31\textwidth]{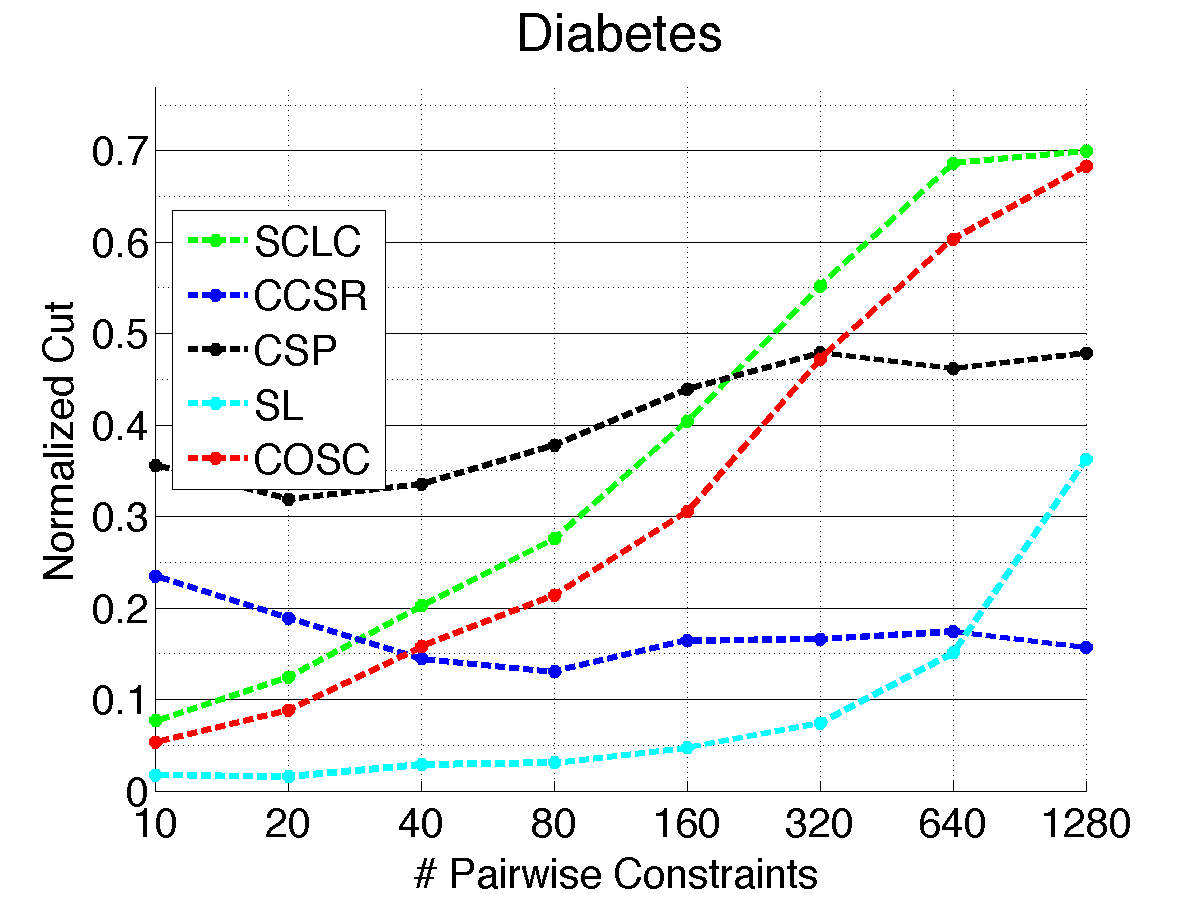}
  	&\includegraphics[width=0.31\textwidth]{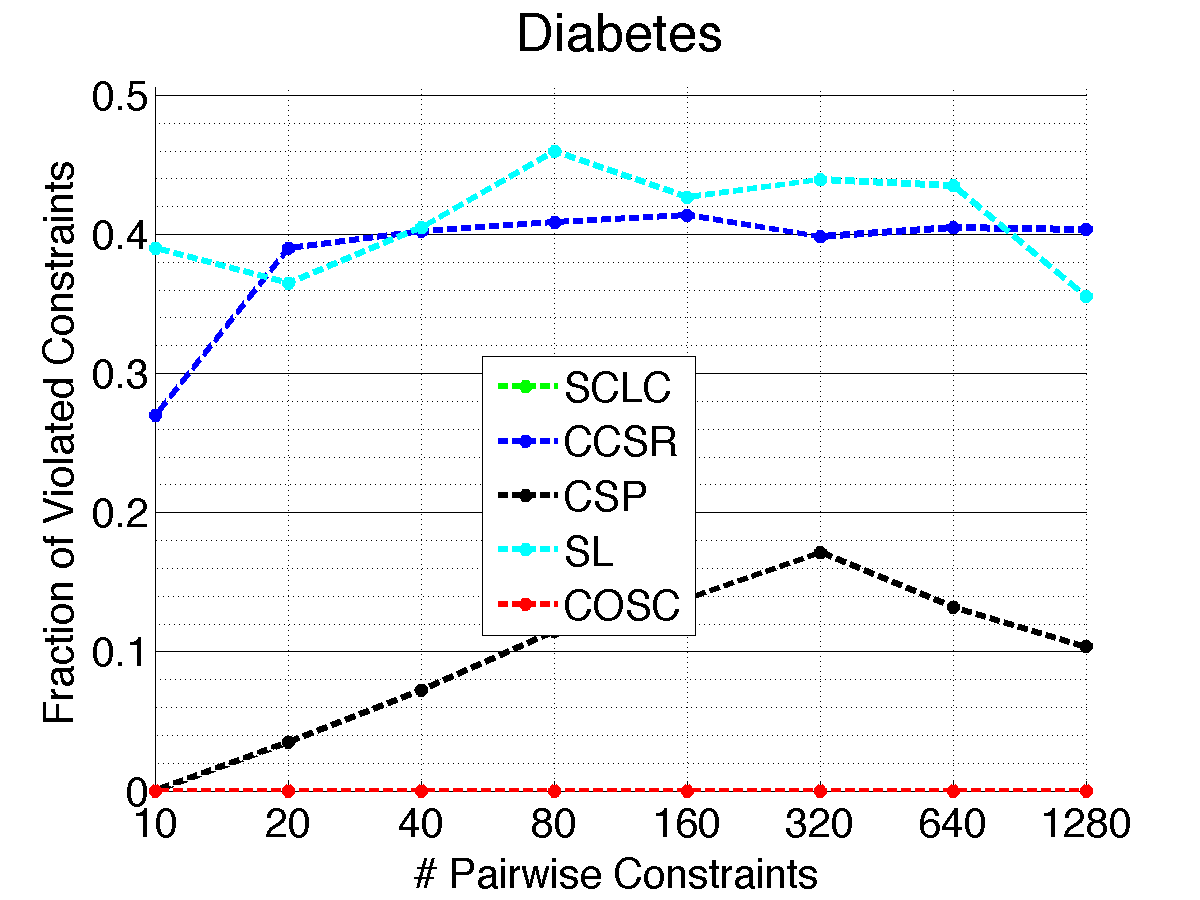}\\
\end{tabular}
\caption{\label{app:tab:Plots}Results for \textbf{binary partitioning}: Left: clustering error versus number of constraints, Middle: normalized cut versus number of constraints, Right: fraction of violated constraints versus number of constraints. }
\end{table*}

 \begin{table*}
\begin{tabular}{ccc}
\includegraphics[width=0.31\textwidth]{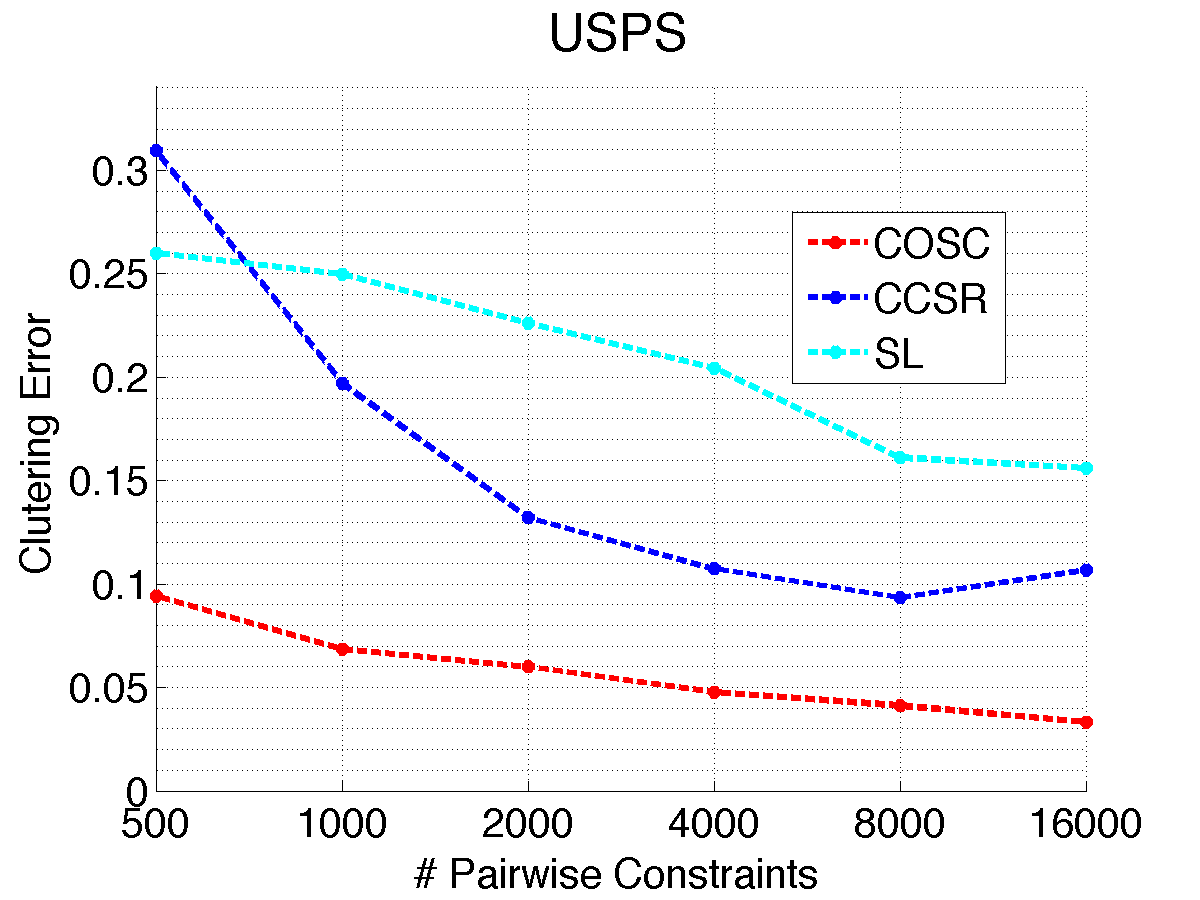}
&  \includegraphics[width=0.31\textwidth]{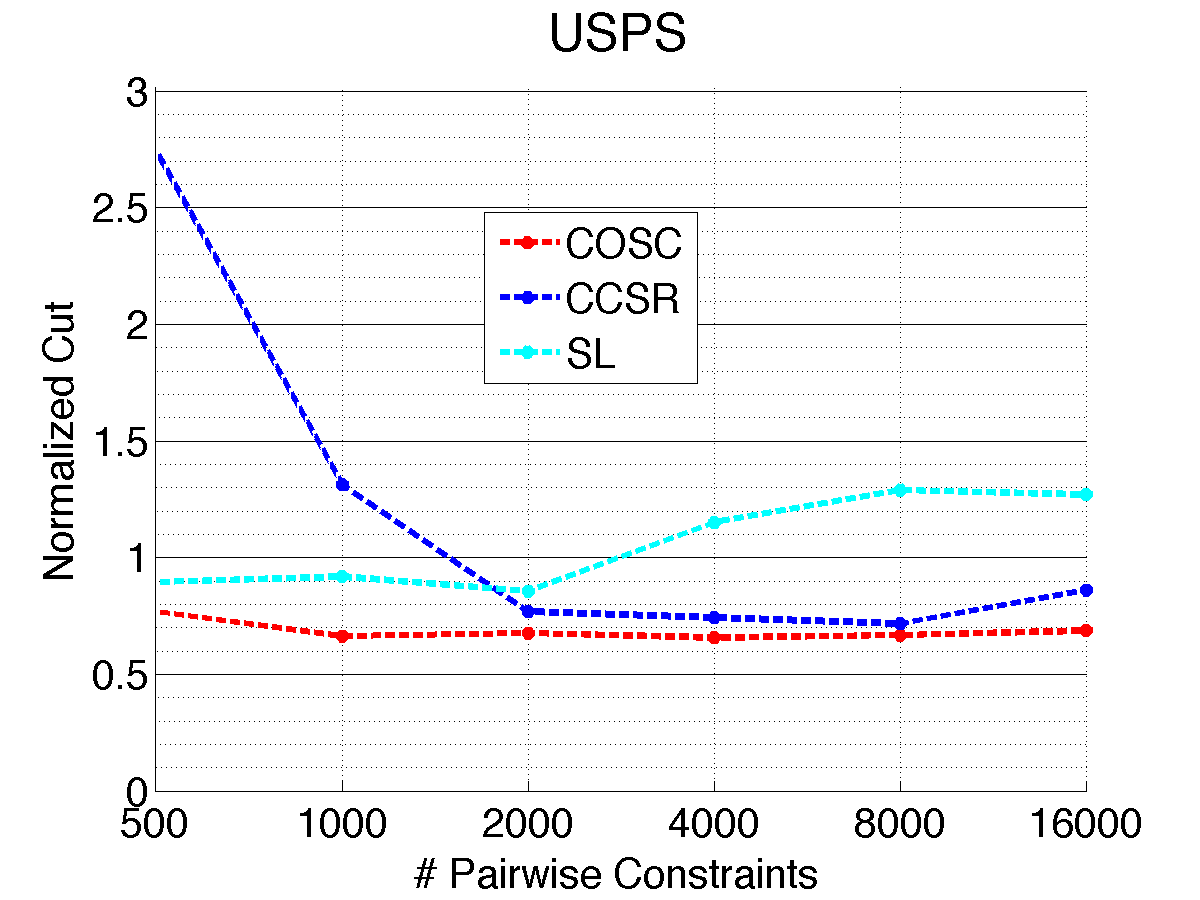}
	&\includegraphics[width=0.31\textwidth]{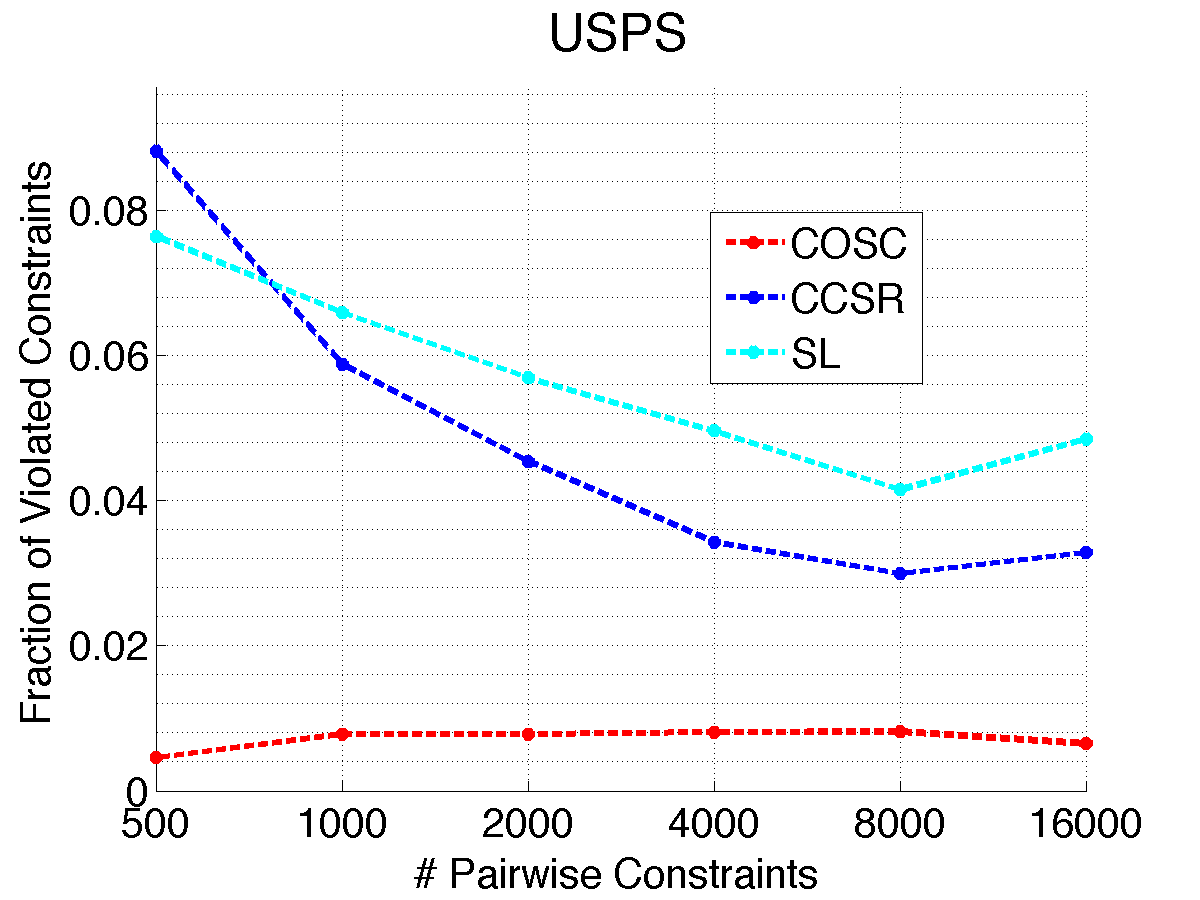}\\
	\includegraphics[width=0.31\textwidth]{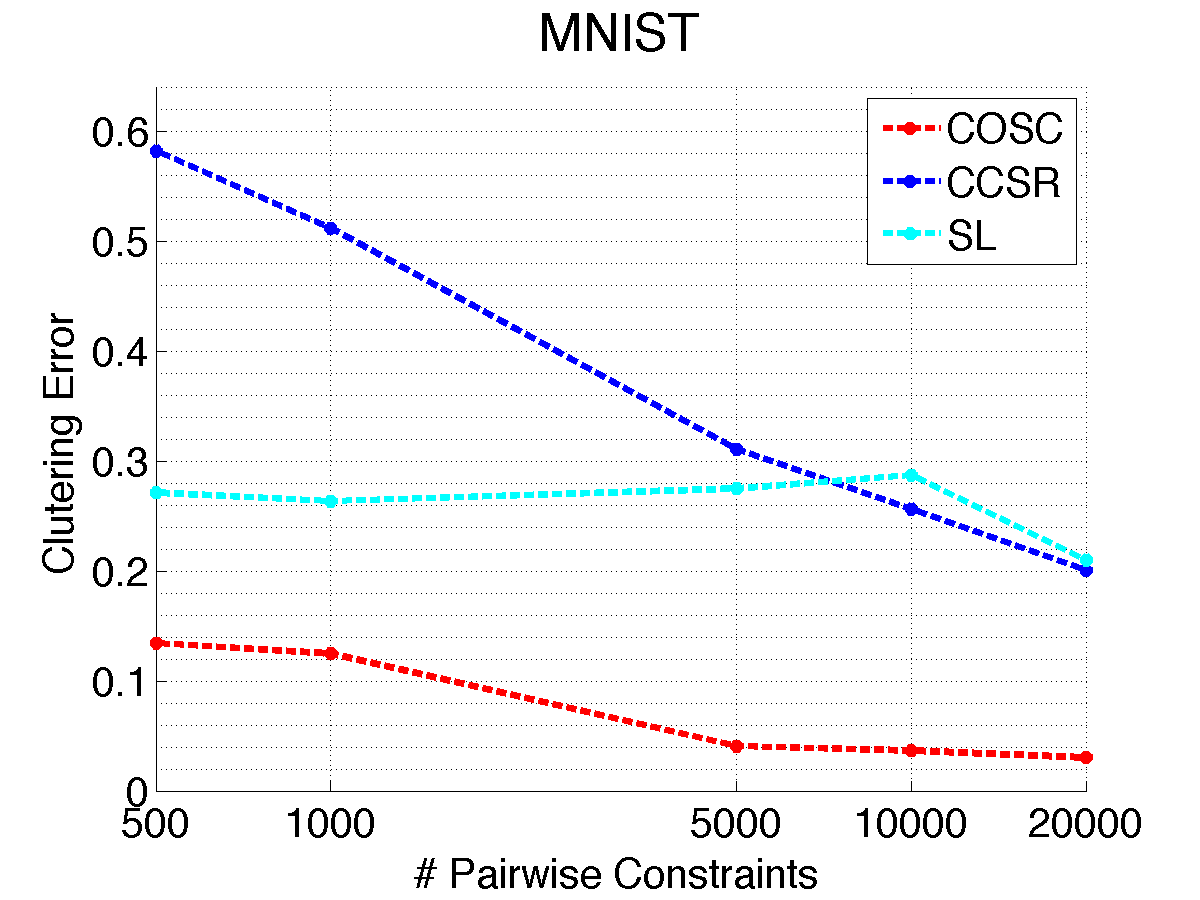}	
	  &\includegraphics[width=0.31\textwidth]{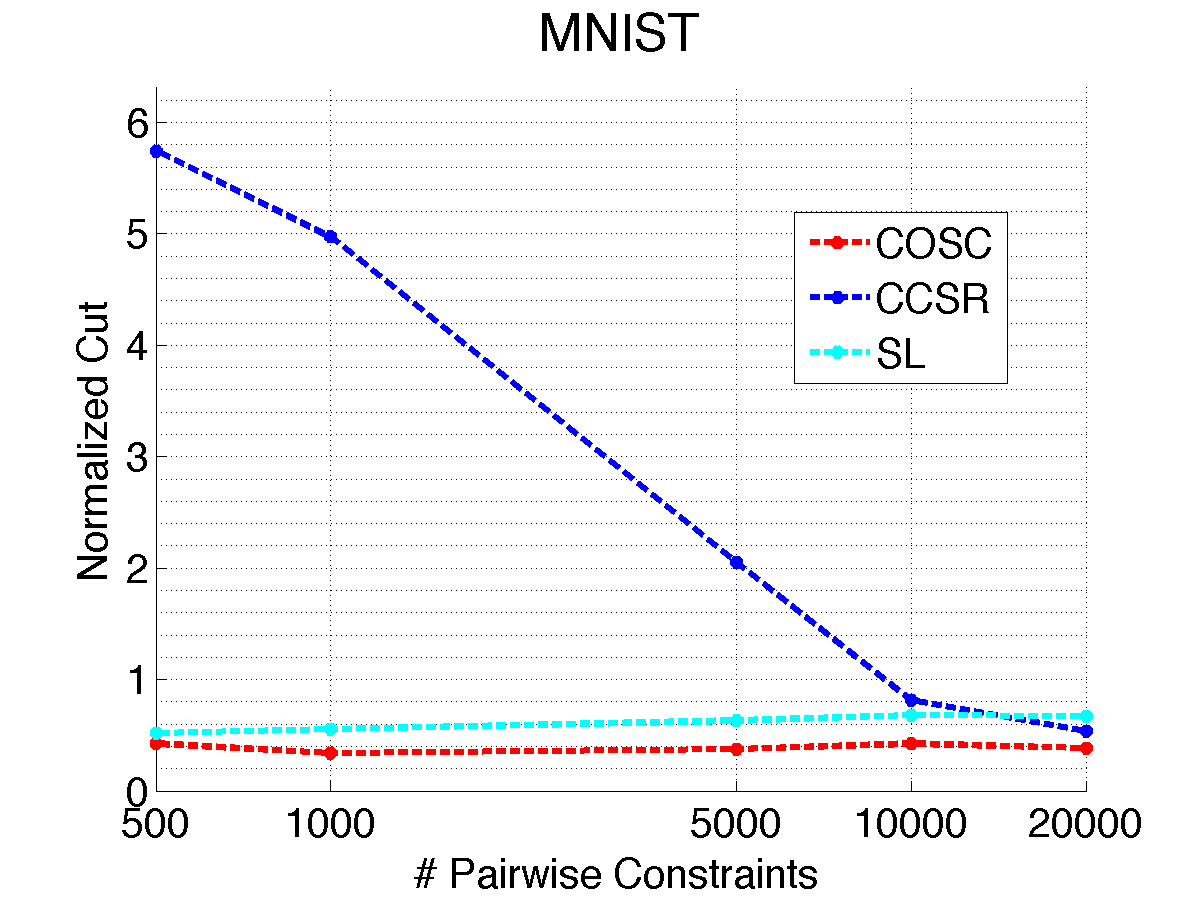}
	&\includegraphics[width=0.31\textwidth]{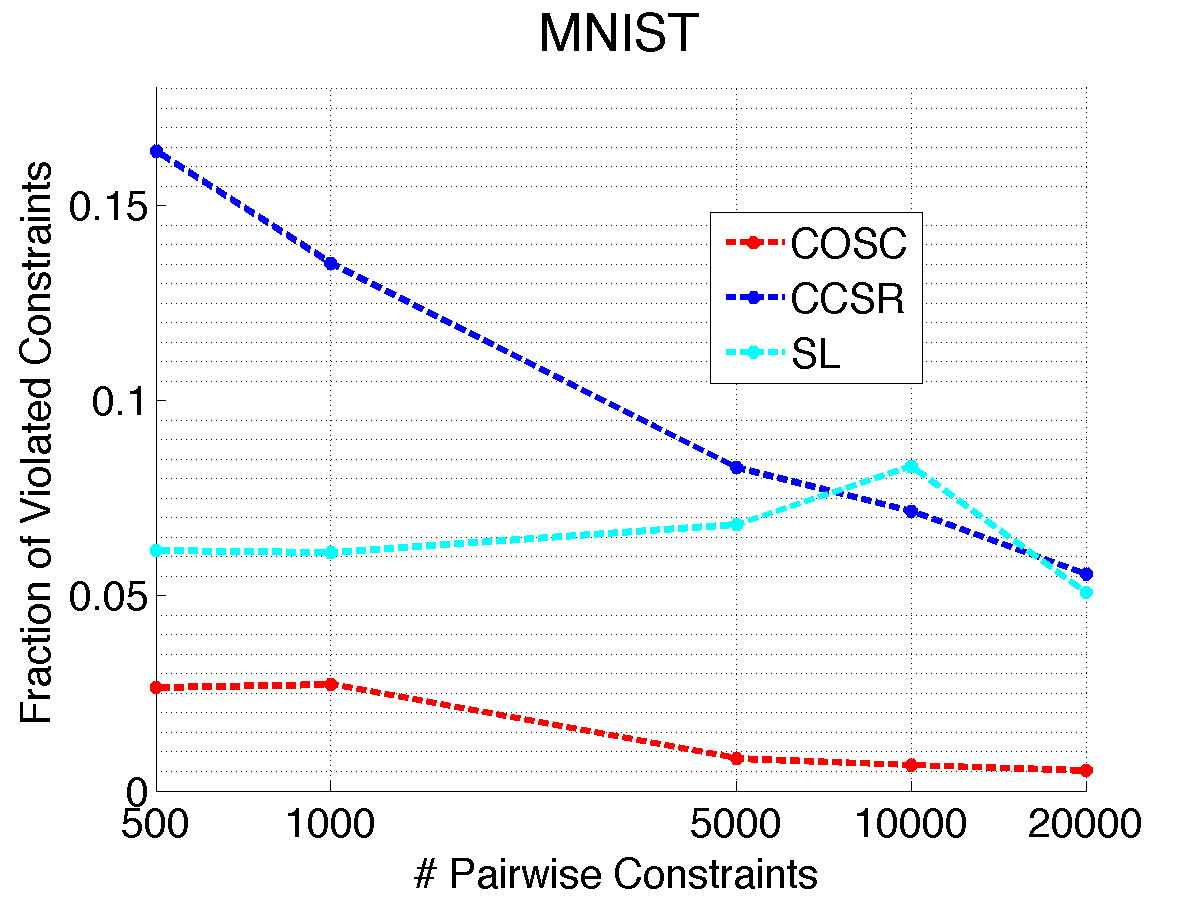}\\
%	\includegraphics[width=0.31\textwidth]{Results/plot_mnist_errors.png}
%  &\includegraphics[width=0.31\textwidth]{Results/plot_mnist_cuts.png}
%	&\includegraphics[width=0.31\textwidth]{Results/plot_mnist_viols.png}\\
		\end{tabular}
\caption{\label{app:tab:Plots_multi}Results for \textbf{multi-partitioning} - Left: clustering error versus number of constraints, Middle: normalized cut versus number of constraints, Right: fraction of violated constraints versus number of constraints.}
\end{table*}

\newpage 
%\subsubsection*{References}
%\bibliographystyle{mlapa}
\bibliographystyle{plain}
\bibliography{cnstr_clustering}

\end{document}